\documentclass{article}
\usepackage{ijcai24}

\usepackage{times}
\usepackage{soul}
\usepackage{url}
\usepackage[hidelinks]{hyperref}
\usepackage[utf8]{inputenc}
\usepackage[small]{caption}
\usepackage{graphicx}
\usepackage{amsmath}
\usepackage{amsthm}
\usepackage{booktabs}
\usepackage[switch]{lineno}

\usepackage[ruled,vlined,linesnumbered]{algorithm2e}

\usepackage{amsmath,amssymb}

\newtheorem{definition}{Definition}
\newtheorem{lemma}{Lemma}
\newtheorem{theorem}{Theorem}

\newtheorem{claim}{Claim}
\newcommand{\stit}{\rightsquigarrow}
\newcommand{\nstit}{\not\rightsquigarrow}
\renewcommand{\O}{{\sf O}}
\renewcommand{\phi}{\varphi}
\renewcommand{\epsilon}{\varepsilon}
\usepackage{booktabs}
\usepackage{multirow}
\usepackage{subcaption}
\newcommand{\SA}{\text{\upshape\fontfamily{qhvc}\selectfont SA}}
\newcommand{\WA}{\text{\upshape\fontfamily{qhvc}\selectfont WA}}
\newcommand{\SE}{\text{\upshape\fontfamily{qhvc}\selectfont SE}}
\newcommand{\WE}{\text{\upshape\fontfamily{qhvc}\selectfont WE}}
\newcommand{\XSTIT}{\text{\upshape\fontfamily{qhvc}\selectfont XSTIT}}
\renewcommand{\O}{{\sf O}}
\newcommand{\X}{{\sf X}}
\renewcommand{\P}{{\sf P}}

\newenvironment{proof-of-claim}{\begin{trivlist}\item\noindent{\em Proof of Claim.}}{\hfill {\tiny $\boxtimes$}\end{trivlist}}

\usepackage{stmaryrd}

\usepackage{multicol,skull}
\usepackage{makecell}
\usepackage{xcolor}
\usepackage{enumitem}
\renewcommand{\[}{\llbracket}
\renewcommand{\]}{\rrbracket}
\newcommand{\llangle}{\langle\!\langle}
\newcommand{\rrangle}{\rangle\!\rangle}

\usepackage{lipsum}
\newcommand\blfootnote[1]{%
  \begingroup
  \renewcommand\thefootnote{}\footnote{#1}%
  \addtocounter{footnote}{-1}%
  \endgroup
}

\title{IJCAI--24 Formatting Instructions}

\author{
    Qi Shi
    \affiliations
    University of Southampton
    \emails
    qi.shi@soton.ac.uk
}

\title{Agentive Permissions in Multiagent Systems}

\begin{document}

\maketitle

\begin{abstract}
This paper proposes to distinguish four forms of agentive permissions in multiagent settings. The main technical results are the complexity analysis of model checking, the semantic undefinability of modalities that capture these forms of permissions through each other, and a complete logical system capturing the interplay between these modalities.\blfootnote{This is the full version of the IJCAI-24 conference paper of the same name. A technical appendix is attached to the end.}
\end{abstract}

\section{Introduction}\label{sec:introduction}

Imagine a large factory being built in a city on a river. The factory will dump a pollutant into the river.
It is known that a small factory located in another city higher up the river already exists and can dump up to $60$g/day of the same pollutant.
Also, more than $100$g/day of the pollutant dumped into the river combined by the two factories will kill the fish.

Suppose that the large factory will dump $20$g/day of the pollutant. Then, the total dumped amount by both factories will not exceed $80$g/day no matter how much the other factory dumps and thus the fish in the river will survive {\em for sure}.
In other words, the action that dumping $20$g/day {\bf\em ensures} the survival of the fish.
On the contrary, if the large factory will dump $60$g/day of the pollutant, then the fish will be killed once the other factory dumps more than $40$g/day.
That is to say, the action that dumping $60$g/day does not ensure the survival of the fish. However, it still {\em leaves the possibility} for the fish to survive, \textit{e.g.} when the other factory dumps no more than $40$g/day.
In this situation, we say that the action that dumping $60$g/day {\bf\em admits} of the survival of the fish.

To ensure and to admit show two different types of {\bf\em agency} of an action. The difference comes from that, in multiagent or non-deterministic settings, the effect of an action of one agent might be affected by the actions of other agents or the nondeterminacy. 
It is worth noting that, admitting is the {\em dual} of ensuring. To be specific, if an action of an agent ensures an outcome, then the action does {\em not} admit of the {\em opposite} outcome, and vice versa. 
It is easy to see that their meanings coincide in single-agent deterministic settings.

In this paper, we consider the two types of agency together with permissions. Let us come back to the large factory. 
Suppose the city government passes a regulation:
\begin{equation}\notag
\begin{minipage}{0.85\linewidth}
     {\em The factory is permitted to dump any amount of the pollutant {as long as} the fish is not killed.}   
\end{minipage}
\end{equation}
This regulation can be interpreted in two ways corresponding to two types of agency.
The first could be
\begin{equation}
\tag{$\SA$}\label{eq:SA}
\begin{minipage}{0.85\linewidth}
    {\em {\bf\em any} dumping amount that {\bf\em admits} of the survival of the fish is permitted.}   
\end{minipage}
\end{equation}
In this interpretation, the permitted dumping amount of the factory is any amount no more than $100$g/day. As long as the factory follows this regulation, there is a chance for the fish to survive if the other factory dumps so little that the total dumped amount does not exceed $100$g/day.
The second interpretation of the regulation could be
\begin{equation}
\tag{$\SE$}\label{eq:SE}
\begin{minipage}{0.85\linewidth}
    {\em {\bf\em any} dumping amount that {\bf\em ensures} the survival of the fish is permitted.}   
\end{minipage}
\end{equation}
In this interpretation, the factory should not dump any amount over $40$g/day and the fish cannot be killed as long as the factory follows the regulation, no matter how much of the pollutant is dumped by the other factory.

The above two interpretations give the factory two different permissions. It is worth noting that, either of the permissions {\em enables} the factory to take any of the actions satisfying some criteria. In this paper, we call such permissions ``strong''.
Specifically, we refer to \eqref{eq:SA} and \eqref{eq:SE} as {\em strong permission to admit} and {\em strong permission to ensure}, respectively.

Not all permissions are in the same form as strong permissions. 
Suppose that, instead of the city's regulation, the factory is under some contractual obligation. To satisfy this obligation, the factory has to dump at least $30$g/day of the pollutant.
In this case, not every amount that ensures/admits of the survival of the fish (\textit{e.g.} $20$g/day) is contractually permitted to be dumped.
Nevertheless, 
\begin{equation}
\tag{$\WE$}\label{eq:WE}
\begin{minipage}{0.85\linewidth}
    {\em {\bf\em there is} a permitted dumping amount that {\bf\em ensures} the survival of the fish.}   
\end{minipage}
\end{equation}
For example, 35g/day is a contractually permitted dumping amount that ensures the survival of the fish.
Similarly, if the contractual obligation forces the factory to dump at least $50$g/day of the pollutant, then no contractually permitted dumping amount ensures the survival of the fish. However, 
\begin{equation}
\tag{$\WA$}\label{eq:WA}
\begin{minipage}{0.85\linewidth}
    {\em {\bf\em there is} a permitted dumping amount that {\bf\em admits} of the survival of the fish.} 
\end{minipage}
\end{equation}
In contrast to strong permissions, the above two permissions express the {\em capability} of the factory to achieve some statements with a permitted action.
We call such permissions ``weak''. Specifically, we refer to \eqref{eq:WE} and \eqref{eq:WA} as {\em weak permission to ensure} and {\em weak permission to admit}, respectively.

As shown in Section~\ref{sec:literature}, the terms ``strong permission'' and ``weak permission'' have already existed in the literature for decades~\cite{raz1975permissions,royakkers1997giving,governatori2013computing}. In Section~\ref{sec:axiom}, we show the consistency between how these terms are used in our paper and in the literature.
However, as far as we see, we are the first to make a clear distinction between ``permission to ensure'' and ``permission to admit'', which are {\bf\em agentive} permissions specific to multiagent settings.
We are also the first to cross-discuss both strong and weak agentive permissions in multiagent settings.

In this paper, we discuss four forms of permissions in multiagent settings that generalise those expressed in statements \eqref{eq:SA}, \eqref{eq:SE}, \eqref{eq:WE}, and \eqref{eq:WA}. 
Our contribution is three-fold.
First, we propose a formal semantics for the four corresponding modalities in multiagent transition systems (Section~\ref{sec:syntax and semantics}). 
We also consider the model-checking problem (Section~\ref{sec:model checking}) and the reduction to STIT logic and ATL (Section~\ref{sec:reduction}).
Second, we prove these modalities are semantically undefinable through each other (Section~\ref{sec:undefinability}). This contrasts to the fact that, when separated from permissions, ensuring and admitting are dual to each other.
Third, we give a sound and complete logical system for the four modalities (Section~\ref{sec:axiom} and Section~\ref{sec:completeness}). This reveals the interplay between the four forms of permissions and offers an efficient way for permission reasoning.

\section{Literature Review and Notion Discussion}\label{sec:literature}

Deontic logic~\cite{mcnamara2022deontic} is an appealing approach to solving AI ethics problems by enabling autonomous agents to comprehend and reason about their \textit{obligation}, \textit{permission}, and \textit{prohibition}.
It aims at ``translating'' the deontic statements in natural languages into logical propositions and building up a system for plausible deduction. 
Von Wright~\shortcite{wright1951deontic} launched the active development of symbolic deontic logic from the analogies between normative and alethic modalities. Several follow-up works~\cite{anderson1956formal,prior1963logic} built up the Standard Deontic Logic (SDL), taking obligation as the basic modality and defining permission as the dual of obligation and prohibition as the obligation of the negation. \citeauthor{anderson1967some}~\shortcite{anderson1967some} and \citeauthor{kanger1971new}~\shortcite{kanger1971new} reduced this system by defining a propositional constant $d$ for ``all (relative) normative demands are met''. By this means, obligation (modality~$\O$) can be defined as
$
\O \phi:=\Box(d\to \phi),
$
which is read as ``it is {\em necessary} that $\phi$ is true when all normative demands are met''. 
As the dual of obligation, permission (modality~$\P$) is defined as
\begin{equation}\label{eq:weak permission def}
\P \phi:=\Diamond(d\wedge\phi),
\end{equation}
which is read as ``it is {\em possible} that all normative demands are met and $\phi$ is true''. In this way, the inference rule
\begin{equation}\label{eq:weak permission ir}
\dfrac{\phi\to\psi}{\P\phi\to\P\psi} 
\end{equation}
is valid. The notion of permission that satisfies statement~\eqref{eq:weak permission ir} is called {\bf weak permission}. There are two well-known related paradoxes about weak permission: 

(\romannumeral1) {\em Ross's paradox} \cite{ross1944imperatives}. The formula 
\begin{equation}\label{eq:weak permission 1}
\P\phi\to\P(\phi\vee\psi)
\end{equation}
is valid by statement~\eqref{eq:weak permission ir}. However, in common sense, for a kid ``it is permitted to eat an apple'' is true but ``it is permitted to eat an apple or drink alcohol'' should be false, which contradicts statement~\eqref{eq:weak permission 1}.

(\romannumeral2) The {\em free choice permission paradox} \cite{wright1968essay,kamp1973free}. 
According to linguistic intuition, if ``it is permitted to eat an apple or a banana'', then both ``eating an apple'' and ``eating a banana'' should be permitted. This shows that disjunctive permission is treated as {\bf free choice permission}, which means the formula
\begin{equation}\label{eq:free choice permission paradox}
\P(\phi\vee\psi)\to\P\phi\wedge\P\psi
\end{equation}
should be valid. However, statement~\eqref{eq:free choice permission paradox} is {\em not} derivable in SDL.
Free choice permission is a form of {\bf strong permission}~\cite{asher2005free}, satisfying the inference rule
\begin{equation}\label{eq:strong permission ir}
\dfrac{\phi\to\psi}{\P^s\psi\to\P^s\phi}.
\end{equation}
Following \citeauthor{anderson1967some} and \citeauthor{kanger1971new}'s way, \citeauthor{benthem1979minimal}~\shortcite{benthem1979minimal} captured the notion of strong permission as 
\begin{equation}\label{eq:strong permission def}
\P^s\phi=\Box(\phi\to d),
\end{equation}
which is read as ``it is {\em necessary} that if $\phi$ is true then all normative demands are met''. He then gave a complete axiom system for obligation ($\O$) and strong permission ($\P^s$).

Most researchers agree that both weak and strong permission makes sense. As discussed by \citeauthor{lewis1979problem}~\shortcite{lewis1979problem}, no universal comprehension of permission seems to exist.
In general, weak permission is treated as the dual of obligation. Strong permission, as well as free choice permission, is more intractable and arouses more interesting discussions due to its anti-monotonic inference property in statement~\eqref{eq:strong permission ir}. For instance, \citeauthor{anglberger2015obligation}~\shortcite{anglberger2015obligation} adopted the notion of strong permission and defined a notion of obligation as the weakest form of (strong) permission. \citeauthor{wang2023strong}~\shortcite{wang2023strong} axiomatised a logic of strong permission that satisfies some commonly desirable logical properties.
Strong permission is also studied in defeasible logic~\cite{asher2005free,governatori2013computing}, which is believed to be able to capture the logical intuition about permission. 
 
The above discussion of permission applies possible-world semantics without specifically considering agents and their agency.
However, it is noticed that two kinds of normative statements exist: the {\em agentless} norms that talk about states (\textit{e.g.} it is permitted to eat an apple) and the {\em agentive} norms that talk about actions (\textit{e.g.} John is permitted to eat an apple).
The possible-world semantics cannot distinguish them.
To fill the gap, \citeauthor{chisholm1964ethics}~\shortcite{chisholm1964ethics} proposed a transfer from any agentive norm to an agentless norm. For instance, the statement ``agent $a$ is permitted to do $\phi$'' is transformed into ``it is permitted that agent $a$ does $\phi$''.
Some recent work \cite{kulicki2017connecting,kulicki2023unified} aimed at integrating the agentless and agentive norms in a unified logical frame.

Things become more complicated when agents and their agency are incorporated.
In the literature, STIT logic~\cite{chellas1969logical,belnap1988seeing,belnap1992way} is used to express the agency. 
\citeauthor{horty1995deliberative}~\shortcite{horty1995deliberative} and \citeauthor{horty2001agency}~\shortcite[chapter 4]{horty2001agency} introduced a deontic STIT logic for {\em ought-to-be} and {\em ought-to-do} semantics, respectively. 
The former corresponds to the agentless obligations while the latter corresponds to the agentive obligations in STIT models.
\citeauthor{horty2001agency}~\shortcite[chapter 3]{horty2001agency} further showed that the transfer proposed by
\citeauthor{chisholm1964ethics} does not always work properly. Following \citeauthor{horty2001agency}, \citeauthor{putte2017free}~\shortcite{putte2017free} briefly discussed the dual of the ought-to-do obligation, which is the weak permission in deontic STIT logic, and then defined a form of free choice permission following statement~\eqref{eq:free choice permission paradox}.
Although the agency is considered in deontic STIT logic, the distinction between to ensure and to admit is never discussed there.

In the field of AI, there is a rising interest in applying deontic logic into agents' planning: how to achieve a goal while complying with the deontic constraints \cite{pandvic2022boid,areces2023deontic}. There is also some discussion of agents' comprehending and reasoning norms \cite{arkoudas2005toward,broersen2018formalising}. However, to the best of our knowledge, the agentive weak and strong permissions have never been cross-discussed before.

In this paper, we consider both permission to ensure and permission to admit in both weak and strong forms that follow statements~\eqref{eq:weak permission def} and \eqref{eq:strong permission def}. 
In a word, we consider four forms of permissions as illustrated in statements \eqref{eq:SA}, \eqref{eq:SE}, \eqref{eq:WE}, and \eqref{eq:WA}.
It is worth mentioning that, our formalisation has a connection with Horty's ought-to-do deontic STIT logic \cite{horty2001agency}. 
On the one hand, the notion ``ensure'' captures the same idea as ``see to it that'' in STIT logic. 
On the other hand, our formalisation can be seen as a {\em reasonable} reduction of Horty's formalisation. 
Specifically, Horty's approach is to, first, define a preference over the outcomes (\textit{i.e.} ``histories'' in STIT models) of actions in the {\em model}, then, apply the {\em dominance act utiliarianism} to decide which actions are permitted (\textit{i.e.} ``optimal'' in his work) in {\em semantics}, and, finally, define the ought-to-do obligation based on the permitted actions in {\em semantics}. 
In particular, an action in the STIT frame is the set of outcomes that may follow from the action. An action ``sees to it that'' $\phi$ if and only if $\phi$ is true in all the potential outcomes. Then, ``do $\phi$'' is interpreted as ``seeing to it that $\phi$''.
In this paper, we combine the first two steps of Horty's approach, directly defining the deontic constraints on actions in the {\em model} and then defining four forms of permissions in {\em semantics}. 
Note that, the definition of deontic notions is independent of the process that combines the first two steps in Horty's approach.
In other words, our work in this paper can easily be transformed from action-based models into outcome-based models by recovering the step of deciding permitted actions based on preference over outcomes using dominance act utilitarianism.
Moreover, we give a reduction of our semantics into STIT logic in Section~\ref{sec:reduction}.

\section{Syntax and Semantics}\label{sec:syntax and semantics}

In this section, we introduce the syntax and semantics of our logical system.
Throughout the paper, unless stated otherwise, we assume a fixed set $\mathcal{A}$ of agents and a fixed nonempty set of propositional variables.

\begin{definition} \label{df:transition system}
A {\bf\em transition system} is a tuple $(S,\Delta,D,M,\pi)$:
\begin{enumerate}
    \item $S$ is a (possibly empty) set of \textbf{states};\label{df:transition system S}
    \item $\Delta=\{\Delta_a^s\}_
    {s\in S,a\in\mathcal{A}}$ is the \textbf{action} space, where $\Delta_a^s$ is a nonempty set of actions available to agent $a$ in state $s$; \label{df:transition Delta}
    \item $D=\{D_a^s\}_{s\in S,a\in\mathcal{A}}$ is the \textbf{deontic constraints}, where $D_a^s$ is a set of permitted actions and $\varnothing\subsetneq D_a^s\subseteq\Delta_a^s$;\label{df:transition system D}
    \item $M=\{M_s\}_{s\in S}$ is the \textbf{mechanism}, where a relation $M_s\subseteq \prod_{a\in\mathcal{A}}\Delta_a^s\times S$ satisfies the {\bf\em continuity} condition: for each \textbf{action profile} $\delta\in \prod_{a\in\mathcal{A}}\Delta_a^s$ there is a state $t\in S$ such that $(\delta,t)\in M_s$;\label{df:transition system M}
    \item $\pi(p)\subseteq S$ for each propositional variable $p$. \label{df:transition system pi}
\end{enumerate}
\end{definition}
The continuity condition in item~\ref{df:transition system M} above requires the existence of a ``next'' state $t$.
We say that a transition system is {\bf\em deterministic} if such state $t$ is always unique.

The language $\Phi$ of our logical system is defined by the following grammar:
$$\phi := p\; |\;\neg\phi\; |\;\phi\vee\phi\; |\;\WA_a\phi\; |\;\WE_a\phi\; |\;\SE_a\phi\; |\;\SA_a\phi,$$
where $p$ is a propositional variable and $a\in\mathcal{A}$ is an agent. 
Intuitively, we interpret $\WA_a\phi$ as ``there is a permitted action of agent $a$ that admits of $\phi$'', 
$\WE_a\phi$ as ``there is a permitted action of agent $a$ that ensures $\phi$'',
$\SE_a\phi$ as ``each action of agent $a$ that ensures $\phi$ is permitted'', 
and
$\SA_a\phi$ as ``each action of agent $a$ that admits of $\phi$ is permitted''.
We assume that conjunction $\wedge$, implication $\to$, and Boolean constants true $\top$ and false $\bot$ are defined in the usual way. Also, by $\bigwedge_{i\le n}\phi_i$ and $\bigvee_{i\le n}\phi_i$ we denote, respectively, the conjunction and the disjunction of the formulae $\phi_1,\dots\phi_n$. As usual, we assume that the conjunction and the disjunction of an empty list are $\top$ and $\bot$, respectively.

\begin{definition} 
\label{df:sat}
For each transition system $(S,\Delta,D,M,\pi)$, each state $s\in S$, and each formula $\phi\in\Phi$, the {\bf\em satisfaction relation} $s\Vdash\phi$ is defined recursively as follows:
\begin{enumerate}
    \item $s\Vdash p$, if $s\in\pi(p)$;\label{item:sat propositional}
    \item $s\Vdash\neg\phi$, if $s\nVdash\phi$;\label{item:sat negation}
    \item $s\Vdash \phi\vee\psi$, if $s\Vdash\phi$ or $s\Vdash\psi$;\label{item:sat disjunction}
    \item $s\Vdash\WA_a\phi$, if $(s,i)\nstit_a\neg\phi$ for some $i\in D_a^s$;\label{item:sat WA}
    \item $s\Vdash\WE_a\phi$, if $(s,i)\stit_a\phi$ for some $i\in D_a^s$;\label{item:sat WE}
    \item $s\Vdash\SE_a\phi$, if $i\in D_a^s$ for each $i$ such that $(s,i)\stit_a\phi$;\label{item:sat SE}
    \item $s\Vdash\SA_a\phi$, if $i\in D_a^s$ for each $i$ such that $(s,i)\nstit_a\neg\phi$,\label{item:sat SA}
\end{enumerate}
where the notation $(s,i)\stit_a\phi$ means that, for each tuple $(\delta,t)\in M_s$, if $\delta_a=i$, then $t\Vdash\phi$.
\end{definition}
Items~4 - 7 above capture the generalised notions of permissions in statements \eqref{eq:WA}, \eqref{eq:WE}, \eqref{eq:SE}, and \eqref{eq:SA} in Section~\ref{sec:introduction}. Informally, $(s,i)\stit_a\phi$ means that that action $i$ of agent $a$ in state $s$ {\em ensures} that $\phi$ is true in the next state. 
Accordingly, $(s,i)\nstit_a\neg\phi$ means that action $i$ of agent $a$ in state $s$ {\em admits} of the situation that $\phi$ is true in the next state.
Observe that, if a transition system is deterministic and has only one agent $a$, then $(s,i)\nstit_a\neg\phi$ if and only if $(s,i)\stit_a\phi$.
Then, the next lemma follows from items 4 - 7 of Definition~\ref{df:sat}.
\begin{lemma}\label{lm:ensure=admit}
If set $\mathcal{A}$ contains only agent $a$, then for any formula $\phi\in\Phi$ and state $s$ of a deterministic transition system,
\begin{enumerate}
    \item $s\Vdash\WA_a\phi$ if and only if $s\Vdash\WE_a\phi$;
    \item $s\Vdash\SA_a\phi$ if and only if $s\Vdash\SE_a\phi$.
\end{enumerate}
\end{lemma}
\noindent Note that, in other cases (\textit{i.e.} multiagent or non-deterministic systems), these modalities are not only semantically {\em inequivalent} but also {\em undefinable} through each other. We show this in Section~\ref{sec:undefinability}.

\subsection{Model Checking}\label{sec:model checking}

Following Definition~\ref{df:sat}, we consider the {\bf\em global} model-checking problem \cite{muller1999model} of language $\Phi$. 
For a finite transition system and a formula $\phi\in\Phi$, the global model checking determines the {\em truth set} $\[\phi\]$ that consists of all states satisfying $\phi$ in the transition system. Formally, we define the truth set of a formula as follows.
\begin{definition}\label{df:truth set}
For any given transition system and any formula $\phi\in\Phi$, the \textbf{truth set} $\[\phi\]$ is the set $\{s\;|\;s\Vdash\phi\}$. 
\end{definition}

The global model checking of formula $\phi$ applies a trivial recursive process on its structural complexity. 
The next theorem shows its time complexity. See Appendix~\ref{app:direct model checking} for the model checking algorithm and detailed analysis.

\begin{theorem}[time complexity]
For a finite transition system $(S,\Delta,D,M,\pi)$ and a formula $\phi\in\Phi$, the time complexity of global model checking is $O\big(|\phi|\cdot(|S|\!+\!|M|\!+\!|\Delta|)\big)$, where $|\phi|$ is the size of the formula, $|S|$ is the number of states, $|M|=\sum_{s\in S}|M_s|$ is the size of the mechanism, and $|\Delta|=\sum_{a\in\mathcal{A}}\sum_{s\in S}|\Delta_a^s|$ is the size of the action space.
\end{theorem}

\subsection{Reduction to Other Logics}\label{sec:reduction}

Recall statements~\eqref{eq:weak permission def} and \eqref{eq:strong permission def} in Section~\ref{sec:literature}, which show the way how \citeauthor{anderson1967some} and \citeauthor{kanger1971new} reduces SDL. 
Using a similar technique, we can translate our modalities into modalities in STIT logic and ATL~\cite{alur2002alternating} after properly interpreting the transition system in Definition~\ref{df:transition system}.

\paragraph{Reduction to STIT}
Instead of being about the states, the statements in STIT logic are about moment-history ($m/h$) pairs. Due to this fact, there is no exact reduction of our logic to STIT logic. However, we can interpret our modalities in STIT models in the appearance of the necessity and possibility modalities $\Box$ and $\Diamond$\footnote{$m/h\Vdash\Box\phi$ iff $m/h'\!\Vdash\phi$ for each history $h'$ such that $m\in h'$.}.
In order to do this, we first incorporate the deontic constraints into the models as atomic propositions. To be specific, $m/h\Vdash d_a$ represents that the action of agent $a$ at moment $m$ that includes history $h$ is permitted. 
We use the modality $\XSTIT$ \cite{broersen2008complete,broersen2011deontic}. Informally, $m/h\Vdash\XSTIT_a\phi$ could be interpreted as ``the action of agent $a$ at moment $m$ that includes history $h$ sees to it that $\phi$ is true at the next moment''.
Then, our four modalities can be translated as:
\begin{equation}\notag
\begin{aligned}
&\WA_a\phi:=\Diamond(d_a\wedge\neg\XSTIT_a\neg\phi);\\
&\WE_a\phi:=\Diamond(d_a\wedge\XSTIT_a\phi);\\
&\SE_a\phi:=\Box(\XSTIT_a\phi\to d_a);\\
&\SA_a\phi:=\Box(\neg\XSTIT_a\neg\phi\to d_a).
\end{aligned}
\end{equation}

\paragraph{Reduction to ATL}
Unlike in STIT logic, in ATL, the statements are about states but there is no way to express the properties of actions.
For this reason, we encode deontic constraints into the states. 
To do this, we expand each state in our original transition system into a set of states in the ATL model.
Specifically, each state $s$ in our original transition system corresponds to a set $\{\langle s,\mathcal{D}\rangle\,|\,\mathcal{D}\subseteq\mathcal{A}\}$ of states in the ATL model.
Informally, the tuple $\langle s,\mathcal{D}\rangle$ encodes the information that ``state $s$ is reached after the agents in set $\mathcal{D}$ taking permitted actions and the others taking non-permitted actions''.
Then, $\langle s,\mathcal{D}\rangle\Vdash d_a$ if and only if $a\in \mathcal{D}$. Also, $\langle s,\mathcal{D}\rangle\Vdash p$ if and only if $s\Vdash p$ in our original transition system.
Correspondingly, the transition $(\langle s,*\rangle,\langle t,\mathcal{D}\rangle)$ exists in the ATL model if there is a tuple $(\delta,t)\in M_s$ in our original transition system such that $\mathcal{D}=\{a\in\mathcal{A}\,|\,\delta_a\in D_a^s\}$ and $*$ is a wildcard.
Note that, ATL requires the transitions to be deterministic.
Thus, if needed, we incorporate a dummy agent $Nature$ into the agent set $\mathcal{A}$ to achieve determinacy.
Then, we can translate our modalities into standard ATL syntax as:
\begin{equation}\notag
\begin{aligned}
&\WA_a\phi:= \llangle\mathcal{A}\rrangle \X(d_a\wedge\phi);\\
&\WE_a\phi:=\llangle a\rrangle \X(d_a\wedge\phi);\\
&\SE_a\phi:=\neg\llangle a\rrangle\X\neg(\phi\to d_a);\\
&\SA_a\phi:=\neg\llangle\mathcal{A}\rrangle\X\neg(\phi\to d_a),
\end{aligned}
\end{equation}
where $\llangle \mathcal{C}\rrangle\X\phi$ is informally interpreted as ``the agents in set $\mathcal{C}$ can cooperate to enforce $\phi$ in the next state'' and $\llangle a\rrangle$ is the abbreviation for $\llangle \{a\}\rrangle$.

\section{Mutual Undefinability}\label{sec:undefinability}

As we define four modalities in the language, we would like to figure out if all of them are necessary to express the corresponding notions of permission. Specifically, if some of these modalities are semantically definable through the others, then the definable ones are not necessary for the language. As an example, a well-known result in Boolean logic is De Morgan's laws, which say conjunction and disjunction are inter-definable in the presence of negation. Therefore, to consider a ``minimal'' system for propositional logic, it is not necessary to include both conjunction and disjunction. 

In this section, we consider the definability of modalities in the same way as De Morgan's laws (\textit{i.e.} semantical equivalence). For example, modality $\WA$ is definable through modalities $\WE$, $\SE$, and $\SA$ if every formula in language $\Phi$ is semantically equivalent to a formula using only modalities $\WE$, $\SE$, and $\SA$. Formally, in the transition systems, we define semantical equivalence as follows.

\begin{definition}\label{df:semantically equivalent}
Formulae $\phi$ and $\psi$ are \textbf{semantically equivalent} if $\[\phi\]=\[\psi\]$ for each transition system.    
\end{definition}

We prove that none of the modalities $\WA$, $\WE$, $\SE$, and $\SA$ is definable through the other three. To do this, it suffices to show that, for each modality $\odot$ of the four modalities, there exists a formula $\odot \phi\in\Phi$ and a transition system where $\[\odot\phi\]\neq \[\psi\]$ for each formula $\psi$ not using modality $\odot$.
In particular, we use the {\em truth set algebra} technique~\cite{knight2022truth}. 
This technique uses one model (\textit{i.e.} transition system) and shows that, in this model, the truth sets of all formulae $\psi$ not using modality $\odot$ form a {\em proper subset} of the family of all truth sets in language $\Phi$, while the truth set of the formula $\odot\phi$ does not belong to this subset.
We formally state the undefinability results in the next theorem.
A detailed explanation of the technique and proof can be found in Appendix~\ref{app:undefinability proofs}.

\begin{theorem}[undefinability of $\WA$]
\label{th:undefinability WA}
The formula $\WA_a p$ is not semantically equivalent to any formula in language $\Phi$ that does not use modality $\WA$. 
\end{theorem}



The formal statements of the undefinability results for modalities $\WE$, $\SE$, and $\SA$ are the same as Theorem~\ref{th:undefinability WA} except for using the corresponding modalities instead of $\WA$, see Theorem~\ref{th:undefinability WE}, Theorem~\ref{th:undefinability SE}, and Theorem~\ref{th:undefinability SA} in Appendix~\ref{app:undefinability proofs}.

Note that, all four undefinability results, as presented in Appendix~\ref{app:undefinability proofs}, require that our language contains at least two agents.
In single-agent settings, if a transition system is non-deterministic, these undefinability results still hold. This can be observed by modifying the two-agent transition systems in the proofs into single-agent non-deterministic transition systems by treating one of the agents as the non-deterministic factor.
If a single-agent transition system is deterministic, then, as observed in Lemma~\ref{lm:ensure=admit}, modalities $\WA$ and $\WE$ are semantically equivalent, so as modalities $\SE$ and $\SA$. However, {\em modalities $\WA$ ($\WE$) and $\SA$ ($\SE$) are not definable through each other}. See Appendix~\ref{app:undefinability single} for proof.

\section{Axioms}\label{sec:axiom}
In addition to the tautologies in language $\Phi$, our logical system contains the following schemes of axioms for all agents $a,b\in\mathcal{A}$ and all formulae $\phi,\psi\in\Phi$:

\begin{enumerate}[label=A\arabic*., ref=A\arabic*,left=0pt]
    \item $\neg\WA_a\bot$; \label{ax: WA bot}
    \item $\WE_a\top$; \label{ax:WE T}
    \item  $\SA_a\bot$;\label{ax: SA bot}
    \item $\SE_a\top\to\SA_a\top$;\label{ax: SA SE T}
    \item $\WA_a(\phi\vee\psi)\to\WA_a\phi\vee\WA_a\psi$;\label{ax: WA disjunction}
    \item $\SA_a\phi\wedge\SA_a\psi\to\SA_a(\phi\vee\psi)$;\label{ax: SA conjunction}
    \item $\WE_a\phi\wedge\neg\WE_a\psi\to\WA_a(\phi\wedge\neg\psi)$;\label{ax: WA WE}
    \item $\neg\SE_a\phi\wedge\SE_a\psi\to\neg\SA_a(\phi\wedge\neg\psi)$;\label{ax: SE SA}
    \item $\neg\WA_a\phi\wedge\SA_a\psi\to\neg\WA_b(\phi\wedge\psi)\wedge\SA_b(\phi\wedge\psi)$.\label{ax: WA SA}
\end{enumerate}

Axiom~\ref{ax: WA bot} says agent $a$ does not have a permitted action that has no next state. This is true because of the continuity property of the mechanism (item~\ref{df:transition system M} of Definition~\ref{df:transition system}). 
Axiom~\ref{ax:WE T} says agent $a$ always has a permitted action that ensures a next state. This is true because of the continuity property and the nonempty set of permitted actions (item~\ref{df:transition system D} of Definition~\ref{df:transition system}). 
Axiom~\ref{ax: SA bot} says every action that may have no next state is permitted. This is true because no such actions exist again due to the continuity property.
Axiom~\ref{ax: SA SE T} is true because both $\SE_a\top$ and $\SA_a\top$ mean that every action of agent $a$ is permitted. 

Axiom~\ref{ax: WA disjunction} says, if agent $a$ has a permitted action that admits of $\phi\vee\psi$, then agent $a$ either has a permitted action that admits of $\phi$ or has a permitted action that admits of $\psi$.
This is true because the permitted action that admits of $\phi\vee\psi$ indeed either admits of $\phi$ or admits of $\psi$ (item~\ref{item:sat disjunction} of Definition~\ref{df:sat}).
Axiom~\ref{ax: SA conjunction} says, if every action of agent $a$ that admits of $\phi$ is permitted and every action of agent $a$ that admits of $\psi$ is permitted, then every action of agent $a$ that admits of $\phi\vee\psi$ is permitted. This is true because any action that admits of $\phi\vee\psi$ either admits of $\phi$ or admits of $\psi$ (item~\ref{item:sat disjunction} of Definition~\ref{df:sat}). 

Axiom~\ref{ax: WA WE} says, if agent $a$ has a permitted action that ensures $\phi$ and has no permitted action that ensures $\psi$, then agent $a$ has a permitted action that admits of $\phi\wedge\neg\psi$. This is true because the permitted action $i$ that ensures $\phi$ does not ensure $\psi$. Hence, action $i$ admits of $\neg\psi$ while $\phi$ is ensured to happen.
Axiom~\ref{ax: SE SA} says, if agent $a$ has a non-permitted action that ensures $\phi$ and every action that ensures $\psi$ is permitted, then agent $a$ has a non-permitted action that admits of $\phi\wedge\neg\psi$. This is true because the non-permitted action $j$ that ensures $\phi$ does not ensure $\psi$. Hence, action $j$ admits of $\neg\psi$ while $\phi$ is ensured to happen.

Axiom~\ref{ax: WA SA} says, if agent $a$ has no permitted action that admits of $\phi$ and every action that admits of $\psi$ is permitted, then agent $b$ has no permitted action that admits of $\phi\wedge\psi$ and every action of agent $b$ that admits of $\phi\wedge\psi$ is permitted. This is true because the antecedent means agent $a$'s permitted actions ensure $\neg\phi$ and non-permitted actions (if existing) ensure $\neg\psi$. Thus, every action of agent $a$ ensures $\neg\phi\vee\neg\psi$ (item~\ref{item:sat disjunction} of Definition~\ref{df:sat}). Hence, $\neg(\phi\wedge\psi)$ is {\em unavoidable} in the next state. This implies that agent $b$ has no permitted action that admits of $\phi\wedge\psi$, and any action of agent $b$ that admits of $\phi\wedge\psi$ is permitted because no such action of agent $b$ exists. 

We write $\vdash\phi$ and say that formula $\phi$ is a {\bf\em theorem} of our logical system if it can be derived from the axioms using the following four inference rules:
\begin{enumerate}[label=IR\arabic*., ref=IR\arabic*,left=0pt]
    \item $\dfrac{\phi\, ,\hspace{3mm}\phi\to\psi}{\psi}$ (Modus Ponens); \label{ir: MP}
    \item $\dfrac{\phi\to\psi}{\WA_a\phi\to\WA_a\psi}$; \label{ir: WA}
    \item $\dfrac{\phi\to\psi}{\SA_a\psi\to\SA_a\phi}$; \label{ir: SA}
    \item $\dfrac{\phi_1\wedge\dots\wedge\phi_m\to\neg\psi_1\vee\dots\vee\neg\psi_n}{\WE_{a_1}\phi_1\wedge\dots\wedge\WE_{a_m}\phi_m\to\SE_{b_1}\psi_1\vee\dots\vee\SE_{b_n}\psi_n}$,\vspace{0.5mm}\\
    where agents $a_1,\dots,a_m,b_1,\dots,b_n$ are distinct. \label{ir: WE SE}
\end{enumerate}

Rule~\ref{ir: WA} is the {\bf monotonicity} rule for modality $\WA$. It is valid because, in each state of each transition system, the permitted action of agent $a$ that admits of $\phi$ also admits of $\psi$, as $\phi\to\psi$ is universally true. By this rule, modality $\WA$ represents a form of {\bf\em weak permission} following statement~\eqref{eq:weak permission ir}.
Rule~\ref{ir: SA} is the {\bf anti-monotonicity} rule for modality $\SA$. It is valid because, in each state of each transition system, the set of actions that admits of $\psi$ is a {\em superset} of the set of actions that admits of $\phi$, as $\phi\to\psi$ is universally true. Hence, as long as the actions in the former set are all permitted, those in the latter set are also permitted. This rule shows that modality $\SA$ represents a form of {\bf\em strong permission} following statement~\eqref{eq:strong permission ir}. 
It can be derived that modality $\WE$ represents a form of weak permission and modality $\SE$ represents a form of strong permission, see Appendix~\ref{app:SE strong ir}.

Rule~\ref{ir: WE SE} is a conflict-preventing rule following the notion of ``ensure'' in semantics. The premise says, if every one of $\phi_1,\dots,\phi_m$ is true, then at least one of $\psi_1,\dots,\psi_n$ is false. The conclusion says, for a set of distinct agents $\{a_1,\dots,a_m,b_1,\dots,b_n\}\subseteq\mathcal{A}$, if every agent $a_i$ has a permitted action to ensure $\phi_i$, then for at least one agent $b_j$, every action that ensures $\psi_j$ is permitted.
This is valid because there would be a conflict otherwise. To be specific, if there is a state $s$ where the conclusion of the rule is false, then, each agent $a_i$ has a permitted action $k_i$ that ensures $\phi_i$ and each agent $b_j$ has a non-permitted action $\ell_j$ that ensures $\psi_j$. Consider an action profile $\delta$ such that $\delta_{a_i}=k_i$ for each $i\leq m$ and $\delta_{b_j}=\ell_j$ for each $j\leq n$. By the continuity condition in item~\ref{df:transition system M} of Definition~\ref{df:transition system}, there is a state $t$ such that $(\delta,t)\in M_s$. In such state $t$, each of the formulae $\phi_1,\dots,\phi_m,\psi_1,\dots,\psi_n$ is true. However, this conflicts with the premise of the rule.

We write $X\vdash\phi$ if a formula $\phi$ can be derived from the {\em theorems} of our logical system and an additional set of assumptions $X$ using {\em only} the Modus Ponens rule. Note that statements $\varnothing\vdash\phi$ and $\vdash\phi$ are equivalent. We say that a set of formulae $X$ is {\em consistent} if $X\nvdash\bot$.

\begin{lemma}[deduction] \label{lm:deduction}
If $X,\phi\vdash \psi$, then $X\vdash\phi\to\psi$.
\end{lemma}
\noindent See the proof of this lemma in Appendix~\ref{app:proof lm deduction}.

\begin{lemma}[Lindenbaum]\label{lm:Lindenbaum}
Any consistent set of formulae can be extended to a maximal consistent set of formulae.
\end{lemma}
\noindent The standard proof of this lemma can be found in \cite[Proposition 2.14]{mendelson2009introduction}. 

\begin{lemma}\label{lm:axiom deduction WE WA}
$\vdash\WE_a\phi\wedge\neg\WA_a\psi\to\WE_a(\phi\wedge\neg\psi)$.
\end{lemma}
\begin{lemma}\label{lm:axiom deduction SE SA}
$\vdash\neg\SE_a\phi\wedge\SA_a\psi\to\neg\SE_a(\phi\wedge\neg\psi)$.
\end{lemma}
\noindent See the proofs of the above two lemmas in Appendix~\ref{app:proof of axioms deduction lemma}.
The next theorem follows from the above discussion of the axioms and the inference rules.

\begin{theorem}[soundness]
If\; $\vdash\phi$, then $s\Vdash\phi$ for each state $s$ of each transition system. 
\end{theorem}

\section{Completeness}\label{sec:completeness}

In this section, we prove the strong completeness of our logical system. As usual, at the core of the completeness theorem is the canonical model construction. In our case, it is a canonical transition system.

\subsection{Canonical Transition System}

In this subsection, we define the canonical transition system $(S,\Delta,D,M,\pi)$.

\begin{definition}\label{df:canonical S}
Set $S$ of states is the family of all maximal consistent subsets of our language $\Phi$.
\end{definition}

For each formula $\phi\in\Phi$, we introduce two actions: a permitted action $\phi^+$ and a non-permitted action $\phi^-$. Formally, by $\phi^+$ and $\phi^-$ we mean the pairs $(\phi,+)$ and $(\phi,-)$, respectively.
By item~\ref{item:sat SA} of Definition~\ref{df:sat}, the formula $\SA_a\top$ expresses that agent $a$ is permitted to use {\em every} action. In other words, if $\SA_a\top$ is true, then there are no non-permitted actions available to agent $a$ in the current state. This explains the intuition behind the following definition.

\begin{definition}\label{df:canonical action space & D}
For each state $s\in S$ and each agent $a\in\mathcal{A}$,
\begin{itemize}
\item[] $\Delta^s_a=
\begin{cases}
    \{\phi^+\;|\;\phi\in\Phi\}, & \text{if\; $\SA_a\top \in s$},\\
    \{\phi^+,\phi^-\;|\;\phi\in\Phi\}, & \text{otherwise;}
\end{cases}$
\item[] $D^s_a=\{\phi^+\;|\;\phi\in \Phi\}.$
\end{itemize}
\end{definition}

The next definition is the key part of the canonical transition system construction. It specifies the mechanism of the transition system. 
Recall that $\neg\WA_a\phi$ means that agent $a$ is not permitted to use any action that admits of $\phi$. 
Hence, each permitted action of $a$ must ensure $\neg\phi$. We capture this rule in item~\ref{dfitem:canonical mechanism 1} of the definition below. 
Recall that $\WE_a\phi$ means that agent $a$ has at least one permitted action that ensures $\phi$. In the canonical transition system, this action is defined to be $\phi^+$. The rule captured in item~\ref{dfitem:canonical mechanism 2} below guarantees $\phi$ in the next state whenever agent $a$ uses action $\phi^+$. 
Next, $\neg\SE_a\phi$ means that agent $a$ is not permitted to use at least one action that ensures $\phi$. We denote such action by $\phi^-$. The rule captured by item~\ref{dfitem:canonical mechanism 3} stipulates that action $\phi^-$ ensures $\phi$. 
Finally, $\SA_a\phi$ means that agent $a$ is permitted to use all actions that admit of $\phi$. In other words, $\SA_a\phi$ means that all non-permitted actions ensure $\neg\phi$. This is captured by item~\ref{dfitem:canonical mechanism 4} below.

\begin{definition}\label{df:canonical M}
$(\delta,t)\in M_s$ when for each agent $a\in\mathcal{A}$ and each formula $\phi\in\Phi$,
\begin{enumerate}
\item if $\delta_a\in D_a^s$ and $\WA_a\phi\notin s$, then $\neg\phi\in t$;\label{dfitem:canonical mechanism 1}
\item if $\delta_a=\phi^+$ and $\WE_a\phi\in s$, then $\phi\in t$;\label{dfitem:canonical mechanism 2}
\item if $\delta_a=\phi^-$ and $\SE_a\phi\notin s$, then $\phi\in t$;\label{dfitem:canonical mechanism 3}
\item if $\delta_a\in \Delta_a^s\setminus D_a^s$ and $\SA_a\phi\in s$, then $\neg\phi\in t$.\label{dfitem:canonical mechanism 4}
\end{enumerate}
\end{definition}
Note that, each state $s$ is a maximal consistent set by Defintion~\ref{df:canonical S}. Hence, for item~1 above, the statement $\WA_a\phi\notin s$ is equivalent to $\neg\WA_a\phi\in s$. The same goes for item~3.


\begin{definition}\label{df:canonical pi}
$\pi(p)=\{s\in S\;|\;p\in s\}$ for each $p$.    
\end{definition}

This concludes the definition of the canonical transition system $(S,\Delta,D,M,\pi)$. Next, we show that it satisfies the {\em continuity} condition in item~4 of Definition~\ref{df:transition system}. 

\begin{lemma}[continuity]
For each state $s\in S$ and each action profile $\delta\in\prod_{a\in\mathcal{A}}\Delta_a^s$, there is a state $t$ such that $(\delta,t)\in M_s$. 
\end{lemma}
\begin{proof}
Consider a partition $\{A,B\}$ of the set $\mathcal{A}$ of agents:
\begin{align}
&A=\{a\in\mathcal{A}\,|\,\delta_{a}\in D_{a}^s\};\label{eq:4-July-11}\\
&B=\{b\in\mathcal{A}\,|\,\delta_{b}\in \Delta_{b}^s\setminus D_{b}^s\}.\label{eq:4-July-12}
\end{align}
Then, $\SA_{b}\top\notin s$ for each agent $b\in B$ by Definitions~\ref{df:canonical action space & D}. Hence, $\neg\SA_{b}\top\in s$ for each agent $b\in B$ because $s$ is a maximal consistent set. 
Then, by the contrapositive of axiom~\ref{ax: SA SE T},
\begin{equation}\label{eq:continuity 0}
\neg\SE_{b}\top\in s.
\end{equation}
Consider the set of formulae
\begin{equation}\label{eq:continuity a}
\begin{aligned}
X=&\{\neg\psi\;|\;\exists\, a\in A \,(\neg\WA_a\psi\in s)\} \\
&\cup \{\sigma\;|\;\exists\, a\in A \,(\delta_a=\sigma^+, \WE_a\sigma\in s)\}\\
&\cup \{\neg\chi\;|\;\exists\, b\in B \,(\SA_b\chi\in s)\}\\
&\cup \{\tau\;|\;\exists\, b\in B \,(\delta_b=\tau^-, \neg\SE_b\tau\in s)\}.
\end{aligned}
\end{equation}

\begin{claim}
Set $X$ is consistent.
\end{claim}
\begin{proof-of-claim}
Suppose the opposite. Then, by axiom~\ref{ax:WE T}, statements~\eqref{eq:continuity 0} and \eqref{eq:continuity a}, there are formulae
\begin{equation}\label{eq:4-July-1}
\begin{aligned}
\hspace{8mm}\neg&\WA_{a_1}\psi_{11},\dots,\neg\WA_{a_1}\psi_{1k_1}, \WE_{a_1}\widehat{\sigma}_1\in s,\\
&\cdots\\
\neg&\WA_{a_m}\psi_{m1},\dots,\neg\WA_{a_m}\psi_{mk_m}, \WE_{a_m}\widehat{\sigma}_m \in s,
\end{aligned}
\end{equation}
and
\begin{equation}\label{eq:4-July-8}
\begin{aligned}
&\SA_{b_1}\chi_{11},\dots,\SA_{b_1}\chi_{1\ell_1}, \neg\SE_{b_1}\widehat{\tau}_1\in s,\\
&\cdots\\
&\SA_{b_n}\chi_{n1},\dots,\SA_{b_n}\chi_{n\ell_n}, \neg\SE_{b_n}\widehat{\tau}_n\in s,
\end{aligned}
\end{equation}
where
\begin{equation}
a_1,\dots,a_m,b_1,\dots,b_n \text{ are distinct agents},\label{eq:4-July-4}
\end{equation}
\begin{equation*}
\widehat{\sigma}_i=
\begin{cases}
    \sigma_i, &\text{if } \delta_{a_i}=\sigma_i^+ \text{ and } \WE_{a_i}\sigma_i\in s,\\
    \top, &\text{otherwise,}
\end{cases}
\end{equation*}
for each $i\leq m$, and
\begin{equation}\notag
\widehat{\tau}_i=
\begin{cases}
    \tau_i,&\text{if } \delta_{b_i}=\tau_i^- \text{ and } \neg\SE_{b_i}\tau_i\in s,\\
    \top, &\text{otherwise,}
\end{cases}
\end{equation}
for each $i\leq n$, such that
\begin{equation}\label{eq:4-July-3}
\bigwedge_{i\leq m}\Big(\widehat{\sigma}_{i}\wedge\bigwedge_{j\leq k_i}\!\!\!\neg\psi_{ij}\Big)\wedge
\bigwedge_{i\leq n}\Big(\widehat{\tau}_i\wedge\bigwedge_{j\leq \ell_i}\!\!\!\neg\chi_{ij}\Big) \vdash \bot.
\end{equation}
By multiple application of the countrapositive of axiom~\ref{ax: WA disjunction} and propositional reasoning, statement~\eqref{eq:4-July-1} implies that
\begin{equation}\label{eq:4-July-2}
s\vdash \neg\WA_{a_i} \Big(\bigvee_{j\leq k_i}\psi_{ij}\Big)\text{ for each $i\leq m$.}
\end{equation}
Note that, in the specific case where $k_i=0$, statement~\eqref{eq:4-July-2} follows directly from axiom~\ref{ax: WA bot}.
By Lemma~\ref{lm:axiom deduction WE WA}, statement~\eqref{eq:4-July-2} and the part $\WE_{a_i}\widehat{\sigma}_i\in s$ of statement~\eqref{eq:4-July-1} imply
\begin{equation}\label{eq:4-July-7}
s\vdash \WE_{a_i}\Big(\widehat{\sigma}_i\wedge\neg\bigvee_{j\leq k_i}\psi_{ij}\Big)\text{ for each $i\leq m$.}
\end{equation}
Meanwhile, by multiple application of Lemma~\ref{lm:deduction} and propositional reasoning, statement~\eqref{eq:4-July-3} can be reformulated to
\begin{equation}\label{eq:4-July-6}
\vdash \bigwedge_{i\leq m}\Big(\widehat{\sigma}_{i}\wedge\neg\bigvee_{j\leq k_i}\psi_{ij}\Big)\to
\bigvee_{i\leq n}\neg\Big(\widehat{\tau}_i\wedge\bigwedge_{j\leq \ell_i}\neg\chi_{ij}\Big).
\end{equation}
  By statement~\eqref{eq:4-July-4} and rule~\ref{ir: WE SE}, statement~\eqref{eq:4-July-6} implies
\begin{equation}\notag
\vdash \bigwedge_{i\leq m}\WE_{a_i}\Big(\widehat{\sigma}_{i}\wedge\neg\bigvee_{j\leq k_i}\psi_{ij}\Big)\to
\bigvee_{i\leq n}\SE_{b_i}\Big(\widehat{\tau}_i\wedge\bigwedge_{j\leq \ell_i}\neg\chi_{ij}\Big).
\end{equation}
Then, by statement~\eqref{eq:4-July-7} and the Modus Ponens rule,
\begin{equation}\label{eq:4-July-10}
s\vdash \bigvee_{i\leq n}\SE_{b_i}\Big(\widehat{\tau}_i\wedge\neg\bigvee_{j\leq\ell_i}\chi_{ij}\Big).
\end{equation}
At the same time, by multiple application of axiom~\ref{ax: SA conjunction} and propositional reasoning, statement~\eqref{eq:4-July-8} implies
\begin{equation}\label{eq:4-July-9}
s\vdash \SA_{b_i}\Big(\bigvee_{j\leq \ell_i}\chi_{ij}\Big)\text{ for each $i\leq n$}.
\end{equation}
Note that, in the specific case where $\ell_i=0$, statement~\eqref{eq:4-July-9} follows directly from axiom~\ref{ax: SA bot}.
By Lemma~\ref{lm:axiom deduction SE SA}, statement~\eqref{eq:4-July-9} and the part $\neg\SE_{b_i}\widehat{\tau}_i\in s$ of statement~\eqref{eq:4-July-8} imply
\begin{equation}\notag
s\vdash \neg\SE_{b_i}\Big(\widehat{\tau}_i\wedge\neg\bigvee_{j\leq \ell_i}\chi_{ij}\Big)\text{ for each $i\leq n$,}
\end{equation}
which contradicts statement~\eqref{eq:4-July-10}.
\end{proof-of-claim}

Let $t$ be any maximal consistent extension of set $X$. By Lemma~\ref{lm:Lindenbaum}, such $t$ must exist. Hence, $t\in S$ by Definition~\ref{df:canonical S}.
\begin{claim}\label{cl:continuity 2}
$(\delta,t)\in M_s$.
\end{claim}
\begin{proof-of-claim}
It suffices to verify that conditions \ref{dfitem:canonical mechanism 1} -- \ref{dfitem:canonical mechanism 4} of Definition~\ref{df:canonical M} are satisfied for the tuple $(\delta,t)$ for each agent $x\in\mathcal{A}$. 
Recall that sets $A$ and $B$ form a partition of the agent set $\mathcal{A}$. Hence, it suffices to consider the following two cases:

\vspace{0.5mm}\noindent{\em Case 1}: $x\in A$. Condition~\ref{dfitem:canonical mechanism 1} of Definition~\ref{df:canonical M} follows from line~1 of statement~\eqref{eq:continuity a} because $X\subseteq t$. Condition~\ref{dfitem:canonical mechanism 2} follows from line~2 of statement~\eqref{eq:continuity a}. Conditions~\ref{dfitem:canonical mechanism 3} and \ref{dfitem:canonical mechanism 4} trivially follow from $\delta_x\in D_x^s$ by statement~\eqref{eq:4-July-11}.

\vspace{0.5mm}\noindent{\em Case 2}: $x\in B$. Conditions~\ref{dfitem:canonical mechanism 1} and \ref{dfitem:canonical mechanism 2} of Definition~\ref{df:canonical M} trivially follow from $\delta_x\notin D_x^s$ by statement~\eqref{eq:4-July-12}. Condition~\ref{dfitem:canonical mechanism 3} follows from line~4 of statement~\eqref{eq:continuity a} because $X\subseteq t$. Condition~\ref{dfitem:canonical mechanism 4} follows from line~3 of statement~\eqref{eq:continuity a}.
\end{proof-of-claim}

The statement of this lemma follows from Claim~\ref{cl:continuity 2}.
\end{proof}

\subsection{Strong Completeness Theorem}\label{sec:completeness theorem}

As usual, at the core of the proof of completeness is a truth lemma proven by induction on the structural complexity of a formula. In our case, it is the Lemma~\ref{lm:truth lemma}. The completeness result, as shown in Theorem~\ref{th:completeness}, is proved with Lemma~\ref{lm:truth lemma} in the standard way. We put the formal proofs in Appendix~\ref{app:proof of truth lemma}.

\begin{lemma}\label{lm:truth lemma}
$s\Vdash \phi$ if and only if $\phi\in s$ for each state $s$ of the canonical transition system and each formula $\phi\in\Phi$.
\end{lemma}

\begin{theorem}[strong completeness]\label{th:completeness}
For each set of formulae $X\subseteq \Phi$ and each formula $\phi\in\Phi$ such that $X\nvdash\phi$, there is a state $s$ of a transition system such that $s\Vdash\chi$ for each $\chi\in X$ and $s\nVdash\phi$.    
\end{theorem}

\section{Conclusion and Future Research}

We are the first to classify the agentive permissions in multiagent settings into permissions to ensure and permissions to admit and cross-discuss them in both weak and strong forms. To do this, we propose and formalise four forms of agentive permissions in multiagent transition systems, analyse the time complexity of the model checking algorithm, prove their semantical undefinability through each other, and give a sound and complete logical system that reveals their interplay.

Future research could be in two directions. One is to extend the deontic constraints from one-step actions to multi-step actions. 
Indeed, multi-step deontic constraints are commonly seen in application scenarios. For example, if a child is permitted to eat only one ice cream per day, then whether to eat an ice cream in the morning affects whether she is permitted to eat one in the afternoon.
This is closely related to conditional norms (\textit{e.g.} conditional obligations) discussed in the literature~\cite{van1973logic,chellas1974conditional,decew1981conditional,rulli2020conditional} but is interpreted in multiagent transition systems instead of possible-world semantics.
The other direction could be the interaction between permission and responsibility in multiagent settings.
It might have been noticed that our introductory example about factories and fish is a variant of \cite[Example 3.11]{halpern2015modification} and \cite[Example 6.2.5]{halpern2016actual}, which talk about causality and responsibility in multiagent settings. Indeed, the connection between obligation and responsibility in linguistic intuition has already been noticed by philosophers \cite{poel2011relation}. However, a formal investigation of the interaction among norm, causation, and responsibility in multiagent settings is still lacking.

\clearpage

\section*{Acknowledgements}
I would like to express my sincere gratitude to Dr. Pavel Naumov, my PhD supervisor, for his selfless help in completing this paper. The research is funded by the China Scholarship Council (CSC No.202206070014).

\bibliography{this}

\clearpage

\appendix

\begin{center}
\bf \huge Technical Appendix
\vspace{2mm}
\end{center}

\section{Model Checking}\label{app:direct model checking}

We consider the global model checking problem in finite transition systems.
Specifically, let us consider an arbitrary finite transition system $(S,\Delta,D,M,\pi)$.
By Definition~\ref{df:sat}, for each formula $\phi\in\Phi$, the calculation of the truth set $\[\phi\]$ uses a \textit{recursive} process on the structural complexity of $\phi$:
\begin{enumerate}
\item if $\phi$ is a propositional variable, then $\[\phi\]=\pi(\phi)$;
\item if $\phi=\neg\psi$, then $\[\phi\]=S\setminus\[\psi\]$;
\item if $\phi=\psi_1\vee\psi_2$, then $\[\psi\]=\[\psi_1\]\cup\[\psi_2\]$;
\item if $\phi=\WA_a\psi$, then $\[\phi\]=\[\WA_a\psi\]$ using Algorithm~\ref{alg:truth set WA};
\item if $\phi=\WE_a\psi$, then $\[\phi\]=\[\WE_a\psi\]$ using Algorithm~\ref{alg:truth set WE};
\item if $\phi=\SE_a\psi$, then $\[\phi\]=\[\SE_a\psi\]$ using Algorithm~\ref{alg:truth set SE};
\item if $\phi=\SA_a\psi$, then $\[\phi\]=\[\SA_a\psi\]$ using Algorithm~\ref{alg:truth set SA}.
\end{enumerate}

Denote by $|S|$ the number of the states in the transition system, by $|\Delta|=\sum_{s\in S}\sum_{a\in\mathcal{A}}|\Delta_a^s|$ the size of the action space, and by $|M|=\sum_{s\in S}|M_s|$ the size of the mechanism.
We analyse the computational {\bf\em time} complexity of model checking.
Note that, for a finite set $S$ of states, a subset of the states can be represented with a Boolean array of the size $O(|S|)$. Then, the set operations that determine whether a set contains an element, add an element to a set, or remove an element from a set take $O(1)$. The other set operations (union, intersection, difference and complement) and determining the subset relation take $O(|S|)$. The same is true for a finite action space.

Since the size of the truth set $\[\phi\]$ is $O(|S|)$ for each formula $\phi\in\Phi$, the computational complexity of each instance of the first three types of recursive steps mentioned above is $O(|S|)$. Next, we analyse the computational complexity of the rest four types of recursive steps.

\textbf{Algorithm~\ref{alg:truth set WA}} shows the pseudocode for calculating the truth set $\[\WA_a\psi\]$, given the transition system, the agent $a$, and the truth set $\[\psi\]$.
Note that, for each state $s\in S$, if a tuple $(\delta,t)\in M_s$ such that $t\Vdash\psi$ exists, then $(s,\delta_a)\nstit_a\neg\psi$ by Definition~\ref{df:sat}; if $\delta_a\in D_a^s$ is also true, then $s\Vdash\WA_a\psi$ by item~\ref{item:sat WA} of Definition~\ref{df:sat}.
Algorithm~\ref{alg:truth set WA} first defines an empty set $Collector$ (line~\ref{alg:line 1-1}) to collect all states satisfying $\WA_a\psi$. Then, for each state $s$ (line~\ref{alg:line 1-2}), the algorithm searches for a tuple $(\delta,t)\in M_s$ (line~\ref{alg:line 1-3}) such that $\delta_a\in D_a^s$ and $t\Vdash\psi$ (line~\ref{alg:line 1-4}). If such a tuple is found, then $s\Vdash\WA_a\psi$. Thus, the algorithm adds state $s$ into the set $Collector$ (line~\ref{alg:line 1-5}) and goes to check the next state (line~\ref{alg:line 1-6}).

Note that, the set $Collector$ defined in line~\ref{alg:line 1-1} is represented by a Boolean array of the size $O(|S|)$. Thus, the execution of line~\ref{alg:line 1-1} takes $O(|S|)$.
Also, the {\bf if} statement in line~\ref{alg:line 1-4} is executed at most $\sum_{s\in S}|M_s|=|M|$ times in total. Each execution of line~\ref{alg:line 1-4} and line~\ref{alg:line 1-5} takes $O(1)$. Hence, the time complexity of Algorithm~\ref{alg:truth set WA} is $O(|S|+|M|)$.

\begin{algorithm}[hbt]
\caption{Calculation of the truth set $\[\WA_a\psi\]$}
\label{alg:truth set WA}
\KwIn{transition system $(S,\Delta,D,M,\pi)$, agent $a$, the truth set $\[\psi\]$}
\KwOut{the truth set $\[\WA_a\psi\]$}
$Collector\leftarrow\varnothing$\label{alg:line 1-1}\;
\For{each state $s\in S$\label{alg:line 1-2}}{
    \For{each tuple $(\delta,t)\in M_s$\label{alg:line 1-3}}{
        \If{$\delta_a\in D_a^s$ and $t\in\[\psi\]$\label{alg:line 1-4}}{
        $Collector$.add($s$)\label{alg:line 1-5}\;
        \textbf{break}\label{alg:line 1-6}\;
        }
    }
}
\Return $Collector$\;
\end{algorithm}

\textbf{Algorithm~\ref{alg:truth set WE}} shows the pseudocode for calculating the truth set $\[\WE_a\psi\]$, given the transition system, the agent $a$, and the truth set $\[\psi\]$.
For each state $s\in S$, if there is a permitted action $i\in D_a^s$ such that $(s,i)\stit_a\psi$, then $s\Vdash\WE_a\psi$ by item~\ref{item:sat WE} of Definition~\ref{df:sat}.
Note that, if there exists a tuple $(\delta,t)\in M_s$ such that $t\nVdash\psi$, then $(s,\delta_a)\nstit_a\psi$ by Definition~\ref{df:sat}.
Algorithm~\ref{alg:truth set WE} first defines an empty set $Collector$ (line~\ref{alg:line 2-1}) to collect all states satisfying $\WE_a\psi$. 
Then, for each state $s$ (line~\ref{alg:line 2-2}), a set $Ensurer$ is initialised to be the set $\Delta_a^s$ (line~\ref{alg:line 2-3}), from which the actions $i\in\Delta_a^s$ such that $(s,i)\nstit_a\psi$ would be removed.
After that, the algorithm checks for each tuple $(\delta,t)\in M_s$ (line~\ref{alg:line 2-4}) if $t\nVdash\psi$ (line~\ref{alg:line 2-5}). If so, then $(s,\delta_a)\nstit_a\psi$. Thus, the algorithm removes action $\delta_a$ from the set $Ensurer$ (line~\ref{alg:line 2-6}).
When the \textbf{for} loop in lines~\ref{alg:line 2-4} - \ref{alg:line 2-6} ends, the set $Ensurer$ consists of all actions $i\in\Delta_a^s$ such that $(s,i)\stit_a\psi$.
Then, the \textbf{if} statement in line~\ref{alg:line 2-7} checks if there is a permitted action $i$ of agent $a$ in state $s$ such that $i\in Ensurer$. If so, then $s\Vdash\WE_a\psi$. Thus, the algorithm adds state $s$ into the set $Collector$ (line~\ref{alg:line 2-8}).

Similar to lines~\ref{alg:line 1-1} and \ref{alg:line 1-4} of Algorithm~\ref{alg:truth set WA}, in Algorithm~\ref{alg:truth set WE}, the execution of line~\ref{alg:line 2-1} and the {\bf if} statement in line~\ref{alg:line 2-5} takes $O(|S|)$ and $O(|M|)$, respectively.
Note that, the set $Ensurer$ defined in line~\ref{alg:line 2-3} is represented by a Boolean array of length $|\Delta_a^s|$. Then, the execution of lines~\ref{alg:line 2-3} and \ref{alg:line 2-7} takes $O(|\Delta_a^s|)$ for each state $s$. Thus, the execution of lines~\ref{alg:line 2-3} and \ref{alg:line 2-7} takes $\sum_{s\in S}O(|\Delta_a^s|)=O(|\Delta|)$ in total.
Meanwhile, line~\ref{alg:line 2-8} are executed $|S|$ times and each execution takes $O(1)$.
Hence, the time complexity of Algorithm~\ref{alg:truth set WE} is $O(|S|\!+\!|M|\!+\!|\Delta|)$.

\begin{algorithm}[htb]
\caption{Calculation of the truth set $\[\WE_a\psi\]$}
\label{alg:truth set WE}
\KwIn{transition system $(S,\Delta,D,M,\pi)$, agent $a$, the truth set $\[\psi\]$}
\KwOut{the truth set $\[\WE_a\psi\]$}
$Collector\leftarrow\varnothing$\label{alg:line 2-1}\;
\For{each state $s\in S$\label{alg:line 2-2}}{
    $Ensurer\leftarrow\Delta_a^s$\label{alg:line 2-3}\;
    \For{each tuple $(\delta,t)\in M_s$\label{alg:line 2-4}}{
        \If{$t\notin\[\psi\]$\label{alg:line 2-5}}{
        $Ensurer$.remove($\delta_a$)\label{alg:line 2-6}\;
        }
    }
    \If{$Ensurer\cap D_a^s\neq\varnothing$\label{alg:line 2-7}}{
        $Collector$.add($s$)\label{alg:line 2-8}\;
    }
}
\Return $Collector$\;
\end{algorithm}

\textbf{Algorithm~\ref{alg:truth set SE}} shows the pseudocode for calculating the truth set $\[\SE_a\psi\]$, given the transition system, the agent $a$, and the truth set $\[\psi\]$.
For each state $s\in S$, if $i\in D_a^s$ for each action $i$ such that $(s,i)\stit_a\psi$, then $s\Vdash\SE_a\psi$ by item~\ref{item:sat SE} of Definition~\ref{df:sat}.
Algorithm~\ref{alg:truth set SE} first defines an empty set $Collector$ (line~\ref{alg:line 3-1}) to collect all states satisfying $\SE_a\psi$. 
Lines~\ref{alg:line 3-2} - \ref{alg:line 3-6} of Algorithm~\ref{alg:truth set SE} are identical to those of Algorithm~\ref{alg:truth set WE}, where the set $Ensurer$ consists of all actions $i\in\Delta_a^s$ such that $(s,i)\stit_a\psi$ when the \textbf{for} loop in lines~\ref{alg:line 3-4} - \ref{alg:line 3-6} ends.
Then, the \textbf{if} statement in line~\ref{alg:line 3-7} checks if every action $i\in Ensurer$ is permitted.
If so, then $s\Vdash\SE_a\psi$. Thus, the algorithm adds state $s$ into the set $Collector$ (line~\ref{alg:line 3-8}).

Due to the similarity between Algorithm~\ref{alg:truth set SE} and Algorithm~\ref{alg:truth set WE}, the computational complexity of Algorithm~\ref{alg:truth set SE} is the same as that of Algorithm~\ref{alg:truth set WE}, which is $O(|S|\!+\!|M|\!+\!|\Delta|)$.

\begin{algorithm}[htb]
\caption{Calculation of the truth set $\[\SE_a\psi\]$}
\label{alg:truth set SE}
\KwIn{transition system $(S,\Delta,D,M,\pi)$, agent $a$, the truth set $\[\psi\]$}
\KwOut{the truth set $\[\SE_a\psi\]$}
$Collector\leftarrow\varnothing$\label{alg:line 3-1}\;
\For{each state $s\in S$\label{alg:line 3-2}}{
    $Ensurer\leftarrow\Delta_a^s$\label{alg:line 3-3}\;
    \For{each tuple $(\delta,t)\in M_s$\label{alg:line 3-4}}{
        \If{$t\notin\[\psi\]$\label{alg:line 3-5}}{
        $Ensurer$.remove($\delta_a$)\label{alg:line 3-6}\;
        }
    }
    \If{$Ensurer\subseteq D_a^s$\label{alg:line 3-7}}{
        $Collector$.add($s$)\label{alg:line 3-8}\;
    }
}
\Return $Collector$\;
\end{algorithm}

\textbf{Algorithm~\ref{alg:truth set SA}} shows the pseudocode for calculating the truth set $\[\SA_a\psi\]$, given the transition system, the agent $a$, and the truth set $\[\psi\]$.
Note that, for each state $s\in S$, if a tuple $(\delta,t)\in M_s$ such that $t\Vdash\psi$ exists, then $(s,\delta_a)\nstit_a\neg\psi$ by Definition~\ref{df:sat}; if $\delta_a\notin D_a^s$ is also true, then $s\nVdash\SA_a\psi$ by item~\ref{item:sat SA} of Definition~\ref{df:sat}.
Algorithm~\ref{alg:truth set SA} first defines a set $Sieve$ and initialises it to be the set of all states (line~\ref{alg:line 4-1}). From the set $Sieve$, each state \textit{not} satisfying $\SA_a\psi$ would be removed. 
Then, for each state $s$ (line~\ref{alg:line 4-2}), the algorithm searches for a tuple $(\delta,t)\in M_s$ (line~\ref{alg:line 4-3}) such that $\delta_a\notin D_a^s$ and $t\Vdash\psi$ (line~\ref{alg:line 4-4}). 
If such a tuple is found, then $s\nVdash\SA_a\psi$. Thus, the algorithm removes $s$ from the set $Sieve$ (line~\ref{alg:line 4-5}) and goes to check the next state (line~\ref{alg:line 4-6}).

Due to the similarity between Algorithm~\ref{alg:truth set SA} and Algorithm~\ref{alg:truth set WA}, it is easy to observe that the time complexity of Algorithm~\ref{alg:truth set SA} is $O(|S|+|M|)$.

\begin{algorithm}[htb]
\caption{Calculation of the truth set $\[\SA_a\psi\]$}
\label{alg:truth set SA}
\KwIn{transition system $(S,\Delta,D,M,\pi)$, agent $a$, the truth set $\[\psi\]$}
\KwOut{the truth set $\[\SA_a\psi\]$}
$Sieve\leftarrow S$\label{alg:line 4-1}\;
\For{each state $s\in S$\label{alg:line 4-2}}{
    \For{each tuple $(\delta,t)\in M_s$\label{alg:line 4-3}}{
        \If{$\delta_a\notin D_a^s$ and $t\in\[\psi\]$\label{alg:line 4-4}}{
        $Sieve$.remove($s$)\label{alg:line 4-5}\;
        \textbf{break}\label{alg:line 4-6}\;
        }
    }
}
\Return $Sieve$\;
\end{algorithm}

To sum up, the model-checking (time) complexity for a formula $\phi\in\Phi$ is $O\big(|\phi|\cdot(|S|\!+\!|M|\!+\!|\Delta|)\big)$, where $|\phi|$ is the size of the formula, $|S|$ is the number of states, $|M|$ is the size of the mechanism, and $|\Delta|$ is the size of the action space.

\section{Proofs of the Undefinability Results}\label{app:undefinability proofs}

\begin{figure*}[hbt]
\begin{center}
\scalebox{.47}{\includegraphics{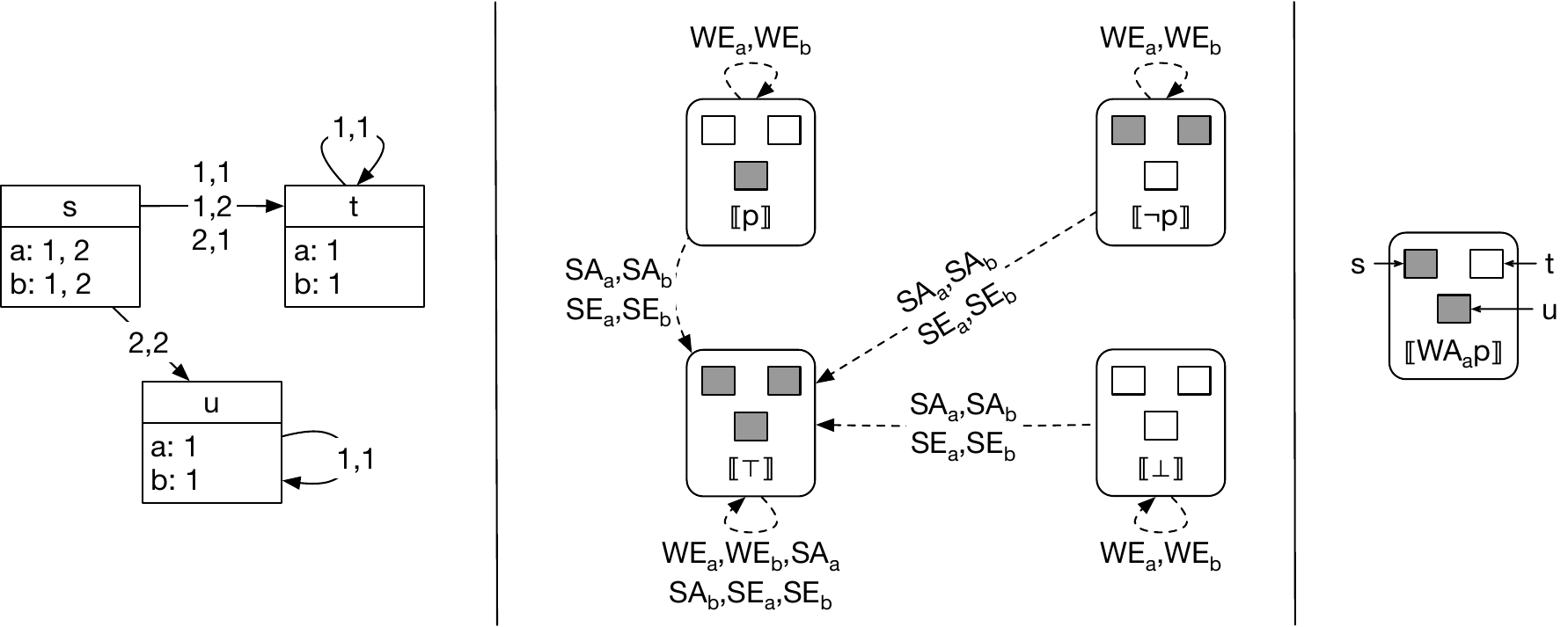}}
\caption{Toward undefinability of $\WA$ through $\WE$, $\SA$, and $\SE$.}
\label{fig:undef WA}
\end{center}
\end{figure*}

We explain the truth set algebra technique while presenting the undefinability result for modality $\WA$ in Subection~\ref{app:undefinability WA}. The proofs of the undefinability results for modalities $\WE$, $\SE$, and $\SA$ go in a similar way as $\WA$ except for using different transition systems. We sketch the proofs of the other three undefinability results in Subsection~\ref{app:undefinability WE}, Subsection~\ref{app:undefinability SE}, and Subsection~\ref{sec:undefinability SA}.

\subsection{Undefinability of $\WA$ Through $\WE$, $\SE$, and $\SA$}\label{app:undefinability WA}

Without loss of generality, in this subsection, we assume that the language has not only at least two agents but {\em exactly} two of them. We refer to these agents as $a$ and $b$. Additionally, also without loss of generality, we assume that the language contains a single propositional variable $p$.

Unlike the well-known {\em bisimulation} technique, which uses two models to prove undefinability, the truth set algebra method uses only one model. In our case, the model is a transition system depicted in the left of Figure~\ref{fig:undef WA}. This transition system has three states: $s$, $t$, and $u$ visualised in the diagram as rectangles. The set $\Delta^y_x$ of actions available to agent $x$ in state $y$ is shown inside the corresponding rectangle. For example, $\Delta^s_a=\{1,2\}$. Note that we use integers to denote actions. 
Uniformly, we use {\em positive} integers to denote the permitted actions in the set $D^y_x$ and {\em negative} integers to denote the non-permitted actions in all the transition systems that will be used to prove our undefinability results.
Specifically, in the transition systems shown in the left of Figure~\ref{fig:undef WA}, every action is permitted.


We denote action profiles by pairs $(i,j)$, where $i$ is the action of agent $a$ and $j$ is the action of agent $b$ under the profile. The mechanism of the transition system is visualised in the left part of Figure~\ref{fig:undef WA} using labelled arrows. For example, the label $(2,1)$ on the arrow from state $s$ to state $t$ denotes the fact that $((2,1),t)\in M_s$. Finally, we assume that $\pi(p)=\{u\}$.

In the left part of Figure~\ref{fig:undef WA}, we visualise the transition system that will be used in the proof. In the middle part and the right part of the same figure, we use the miniaturisation of the transition system (\textit{i.e.} diagram) to visualise a subset of the set of states.
For example, in the diagram in the right part of Figure~\ref{fig:undef WA}, we visualise the set $\{s,u\}$ of states by shading grey the two rectangles corresponding to states $s$ and $u$.
Specifically, we use the diagrams to denote some truth sets as shown at the bottom of each diagram.

Let us now consider the family of four truth sets $\mathcal{F}=\{\[p\],\[\neg p\],\[\top\],\[\bot\]\}$ as visualised in the four diagrams in the middle part of Figure~\ref{fig:undef WA}.
Intuitively, the next lemma shows that, for the above-described transition system, the family $\mathcal{F}$ is closed with respect to modalities $\WE$, $\SE$, and $\SA$.

\begin{lemma}\label{lm:undefinability WA 1}
$\[\WE_x\phi\],\[\SE_x\phi\],\[\SA_x\phi\]\in\mathcal{F}$
for any agent $x\in\{a,b\}$
and any formula $\phi\in\Phi$
such that $\[\phi\]\in\mathcal{F}$.
\end{lemma}
\begin{proof}
We first show that if $\[\phi\]=\[p\]$, then $\[\WE_a\phi\]=\[p\]$. Indeed, $\[\phi\]=\{u\}$. 

Note that $D_a^s=\{1,2\}$. When taken by agent $a$ in state $s$, both actions $1$ and $2$ allow $t$ to be the next state, see Figure~\ref{fig:undef WA}. Then, $(s,1)\nstit_a\phi$ and $(s,2)\nstit_a\phi$ as $\[\phi\]=\{u\}$.
Hence, $s\nVdash \WE_a\phi$ by item~\ref{item:sat WE} of Definition~\ref{df:sat} and that $D_a^s=\{1,2\}$.
One can similarly know that $t\notin\[\WE_a\phi\]$.

On the contrary, the only action available to agent $a$ in state $u$, action $1$, guarantees that the next state is again $u$. Then, $(u,1)\stit_a\phi$ for the action $1\in D^u_a$ as $\[\phi\]=\{u\}$. Hence, $u\Vdash \WE_a\phi$ by item~\ref{item:sat WE} of Definition~\ref{df:sat}. Therefore, $u\in \[\WE_a\phi\]$ by Definition~\ref{df:truth set}.

In conclusion, if $\[\phi\]=\[p\]$, then $\[\WE_a\phi\]=\{u\}=\[p\]$. We visualise this observation by a dashed arrow labelled with $\WE_a$ in the middle part of Figure~\ref{fig:undef WA} from the diagram $\[p\]$ to itself.

The proofs for the other 23 cases are similar. We show the corresponding labelled dashed arrows in the middle part of Figure~\ref{fig:undef WA}. 
\end{proof}

\begin{lemma}\label{lm:undefinability WA 2}
$\[\phi\]\in \mathcal{F}$ for each formula $\phi\in\Phi$ that does not contain modality $\WA$.
\end{lemma}
\begin{proof}
We prove the statement of the lemma by induction on the structural complexity of formula $\phi$. If formula $\phi$ is the propositional variable $p$, then the statement of the lemma holds because $\[p\]\in\mathcal{F}$.

Suppose that formula $\phi$ has the form $\neg \psi$. Then, $\[\phi\]=\{s,t,u\}\setminus \[\psi\]$ by item~\ref{item:sat negation} of Definition~\ref{df:sat} and Definition~\ref{df:truth set}. Observe that family $\mathcal{F}$ is closed with respect to the complement, see the middle part of Figure~\ref{fig:undef WA}. Thus, if $\[\psi\]\in \mathcal{F}$, then $\[\phi\]\in \mathcal{F}$. Therefore, by the induction hypothesis, $\[\phi\]\in \mathcal{F}$. 

Assume that formula $\phi$ has the form $\psi_1\vee \psi_2$. Then, $\[\phi\]=\[\psi_1\]\cup\[\psi_2\]$ by item~\ref{item:sat disjunction} of Definition~\ref{df:sat} and Definition~\ref{df:truth set}. Observe that family $\mathcal{F}$ is closed with respect to the union, see the middle part of Figure~\ref{fig:undef WA}. Thus, if $\[\psi_1\],\[\psi_2\]\in \mathcal{F}$, then $\[\phi\]\in \mathcal{F}$. Therefore, by the induction hypothesis, $\[\phi\]\in \mathcal{F}$.

If formula $\phi$ has one of the forms $\WE_a\psi$, $\WE_b\psi$, $\SE_a\psi$, $\SE_b\psi$, $\SA_a\psi$, or $\SA_b\psi$, then the statement of the lemma follows from the induction hypothesis by Lemma~\ref{lm:undefinability WA 1}. 
\end{proof}

\begin{lemma}\label{lm:undefinability WA 3}
$\[\WA_a p\]\notin \mathcal{F}$.
\end{lemma}
\begin{proof}
The truth set $\[\WA_a p\]$ is depicted at the right of Figure~\ref{fig:undef WA}. This could be verified using an argument similar to the one used in the proof of Lemma~\ref{lm:undefinability WA 1}.    
\end{proof}

Theorem~\ref{th:undefinability WA} follows from Definition~\ref{df:semantically equivalent} and the two previous lemmas.

\noindent\textbf{Theorem~\ref{th:undefinability WA}} (undefinability of $\WA$)
\textit{The formula $\WA_a p$ is not semantically equivalent to any formula in language $\Phi$ that does not use modality $\WA$. }

\begin{figure*}[hbt]
\begin{center}
\scalebox{.47}{\includegraphics{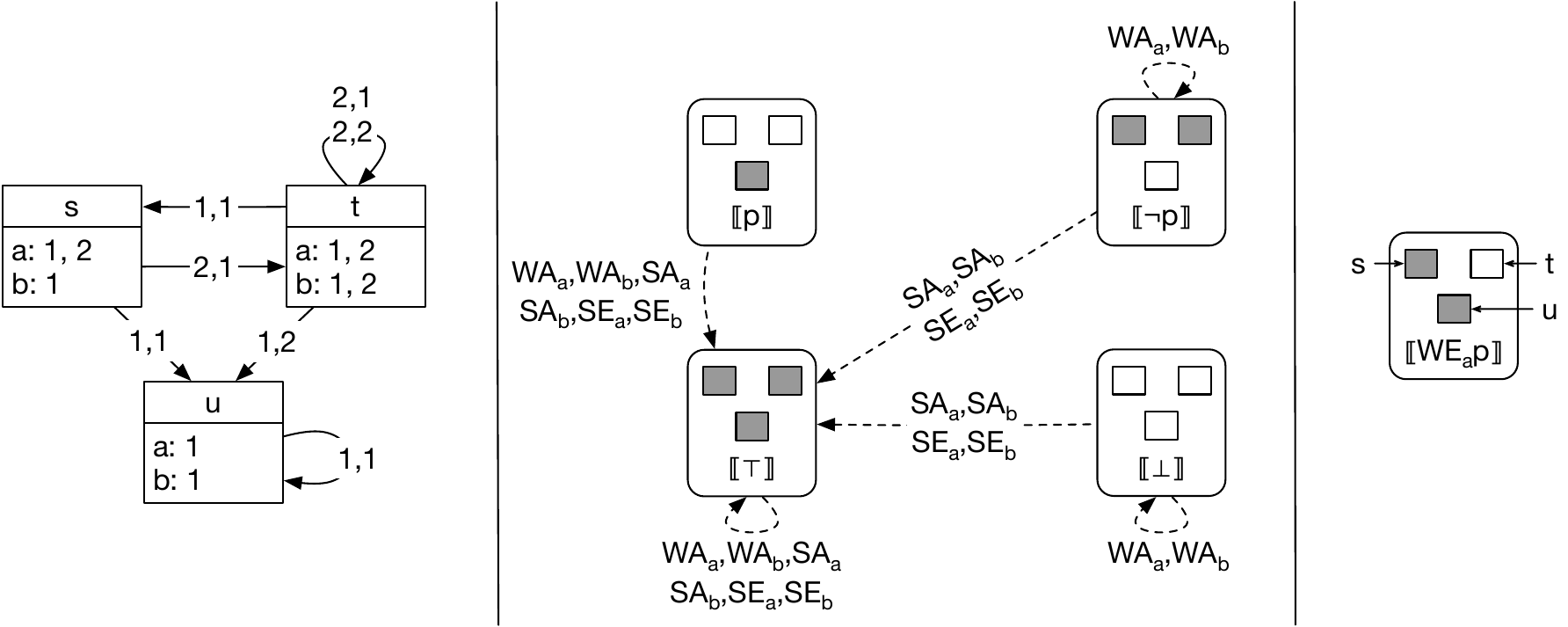}}
\caption{Toward undefinability of $\WE$ through $\WA$, $\SA$, and $\SE$.}
\label{fig:undef WE}
\end{center}
\end{figure*}

\begin{figure*}[ht]
\begin{center}
\scalebox{.47}{\includegraphics{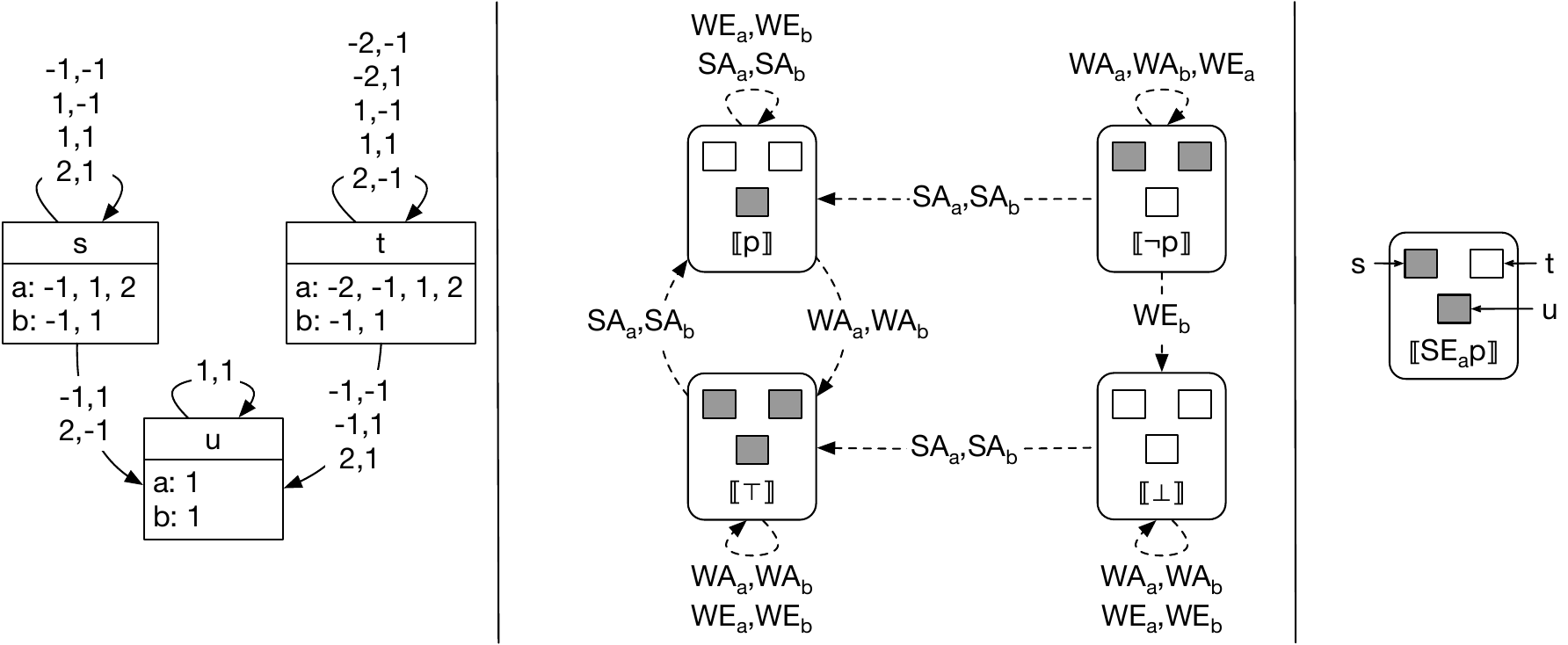}}
\caption{Toward undefinability of $\SE$ through $\WA$, $\WE$, and $\SA$.}
\label{fig:undef SE}
\end{center}
\end{figure*}

\begin{figure*}[ht]
\begin{center}
\scalebox{.47}{\includegraphics{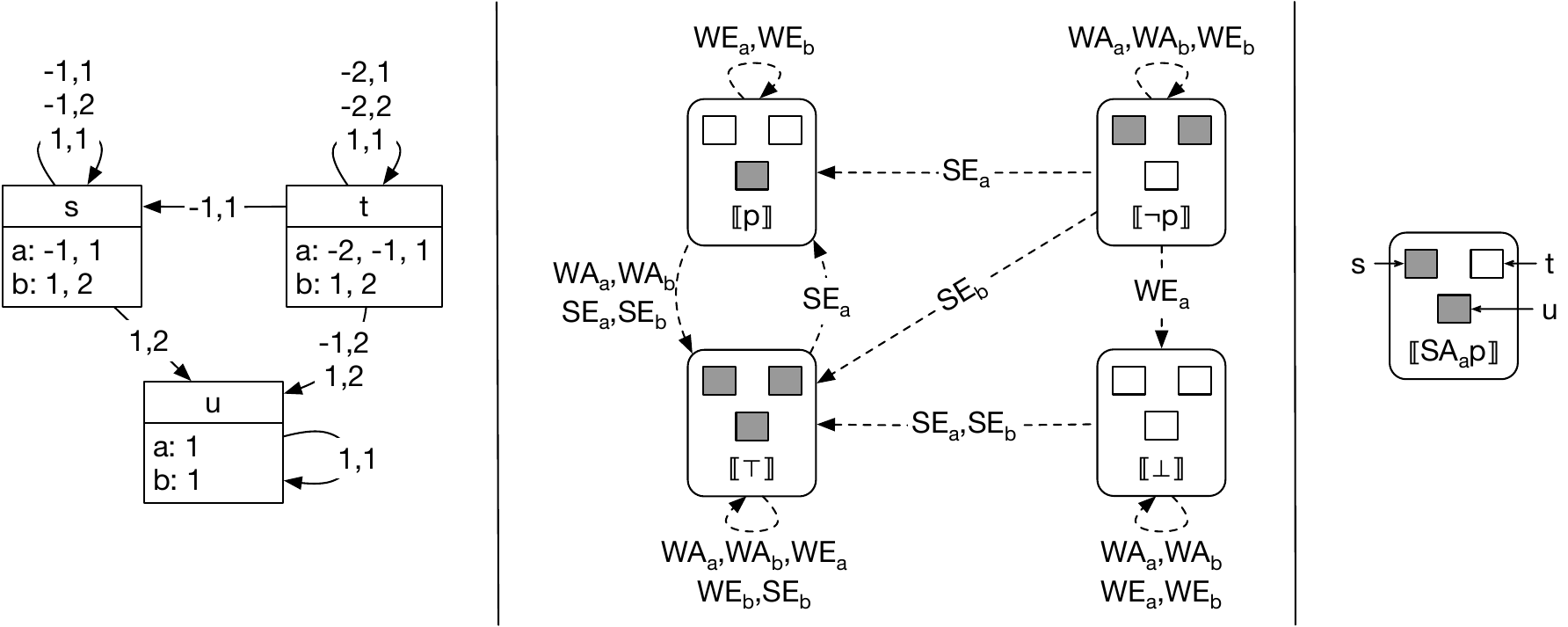}}
\caption{Toward undefinability of $\SA$ through $\WA$, $\WE$, and $\SE$.}
\label{fig:undef SA}
\end{center}
\end{figure*}

\subsection{Undefinability of $\WE$ Through $\WA$, $\SE$, and $\SA$}\label{app:undefinability WE}

In this case, we use the transition system depicted in the left of Figure~\ref{fig:undef WE}. 
We depict the description of the transition system in the same way as in Figure~\ref{fig:undef WA}. Note that, the permitted actions are still denoted by positive integers.
The family of truth sets to consider here is $\mathcal{F}=\{\[p\],\[\neg p\],\[\top\],\[\bot\]\}$, as visualised in the middle of Figure~\ref{fig:undef WE}.

The next three lemmas can be proved in a similar way as Lemma~\ref{lm:undefinability WA 1}, Lemma~\ref{lm:undefinability WA 2}, and Lemma~\ref{lm:undefinability WA 3} except for using Figure~\ref{fig:undef WE} instead of Figure~\ref{fig:undef SA}.
\begin{lemma}\label{lm:undefinability WE 1}
$\[\WA_x\phi\],\[\SE_x\phi\],\[\SA_x\phi\]\in\mathcal{F}$
for any agent $x\in\{a,b\}$
and any formula $\phi\in\Phi$
such that $\[\phi\]\in\mathcal{F}$.
\end{lemma}

\begin{lemma}\label{lm:undefinability WE 2}
$\[\phi\]\in \mathcal{F}$ for each formula $\phi\in\Phi$ that does not contain modality $\WE$.
\end{lemma}

\begin{lemma}\label{lm:undefinability WE 3}
$\[\WE_a p\]\notin \mathcal{F}$.
\end{lemma}
The next theorem follows from Definition~\ref{df:semantically equivalent}, Lemma~\ref{lm:undefinability WE 2}, and Lemma~\ref{lm:undefinability WE 3}.

\begin{theorem}[undefinability of $\WE$]\label{th:undefinability WE}
The formula $\WE_a p$ is not semantically equivalent to any formula in language $\Phi$ that does not use modality $\WE$.  
\end{theorem}

\subsection{Undefinability of $\SE$ Through $\WA$, $\WE$, and $\SA$}\label{app:undefinability SE}

In this case, we use the transition system depicted in the left of Figure~\ref{fig:undef SE}. 
Note that, in this transition system, not all actions are permitted. In particular, the permitted actions are denoted by positive integers and the non-permitted actions are denoted by negative integers.
It is worth mentioning that, by the continuity condition in item~\ref{df:transition system M} of Definition~\ref{df:transition system}, a next state always exists even if the agents take non-permitted actions.
The family of truth sets to consider here is $\mathcal{F}=\{\[p\],\[\neg p\],\[\top\],\[\bot\]\}$, as visualised in the middle of Figure~\ref{fig:undef SE}.

The next three lemmas can be proved in a similar way as Lemma~\ref{lm:undefinability WA 1}, Lemma~\ref{lm:undefinability WA 2}, and Lemma~\ref{lm:undefinability WA 3} except for using Figure~\ref{fig:undef SE} instead of Figure~\ref{fig:undef WA}.
\begin{lemma}\label{lm:undefinability SE 1}
$\[\WA_x\phi\],\[\WE_x\phi\],\[\SA_x\phi\]\in\mathcal{F}$
for any agent $x\in\{a,b\}$
and any formula $\phi\in\Phi$
such that $\[\phi\]\in\mathcal{F}$.
\end{lemma}

\begin{lemma}\label{lm:undefinability SE 2}
$\[\phi\]\in \mathcal{F}$ for each formula $\phi\in\Phi$ that does not contain modality $\SE$.
\end{lemma}

\begin{lemma}\label{lm:undefinability SE 3}
$\[\SE_a p\]\notin \mathcal{F}$.
\end{lemma}
The next theorem follows from Definition~\ref{df:semantically equivalent}, Lemma~\ref{lm:undefinability SE 2}, and Lemma~\ref{lm:undefinability SE 3}.

\begin{theorem}[undefinability of $\SE$]\label{th:undefinability SE}
The formula $\SE_a p$ is not semantically equivalent to any formula in language $\Phi$ that does not use modality $\SE$.  
\end{theorem}

\subsection{Undefinability of $\SA$ Through $\SE$, $\WA$, and $\WE$}
\label{sec:undefinability SA}

In this case, we use the transition system depicted in the left part of Figure~\ref{fig:undef SA}. 
We consider the family of truth sets $\mathcal{F}=\{\[p\],\[\neg p\],\[\top\],\[\bot\]\}$, which are visualised in the middle of Figure~\ref{fig:undef SA}.

The next three lemmas can be proved in a similar way as Lemma~\ref{lm:undefinability WA 1}, Lemma~\ref{lm:undefinability WA 2}, and Lemma~\ref{lm:undefinability WA 3} except for using Figure~\ref{fig:undef SA} instead of Figure~\ref{fig:undef WA}.

\begin{lemma}\label{lm:undefinability SA 1}
$\[\WA_x\phi\],\[\WE_x\phi\],\[\SE_x\phi\]\in\mathcal{F}$
for any agent $x\in\{a,b\}$
and any formula $\phi\in\Phi$
such that $\[\phi\]\in\mathcal{F}$.
\end{lemma}

\begin{lemma}\label{lm:undefinability SA 2}
$\[\phi\]\in \mathcal{F}$ for each formula $\phi\in\Phi$ that does not contain modality $\SA$.
\end{lemma}

\begin{lemma}\label{lm:undefinability SA 3}
$\[\SA_a p\]\notin \mathcal{F}$.
\end{lemma}

The next result follows from Definition~\ref{df:semantically equivalent}, Lemma~\ref{lm:undefinability SA 2}, and Lemma~\ref{lm:undefinability SA 3}.
\begin{theorem}[undefinability of $\SA$]\label{th:undefinability SA}
The formula $\SA_a p$ is not semantically equivalent to any formula in language $\Phi$ that does not use modality $\SA$.  
\end{theorem}

\begin{figure*}[hbt]
\begin{center}
\scalebox{.47}{\includegraphics{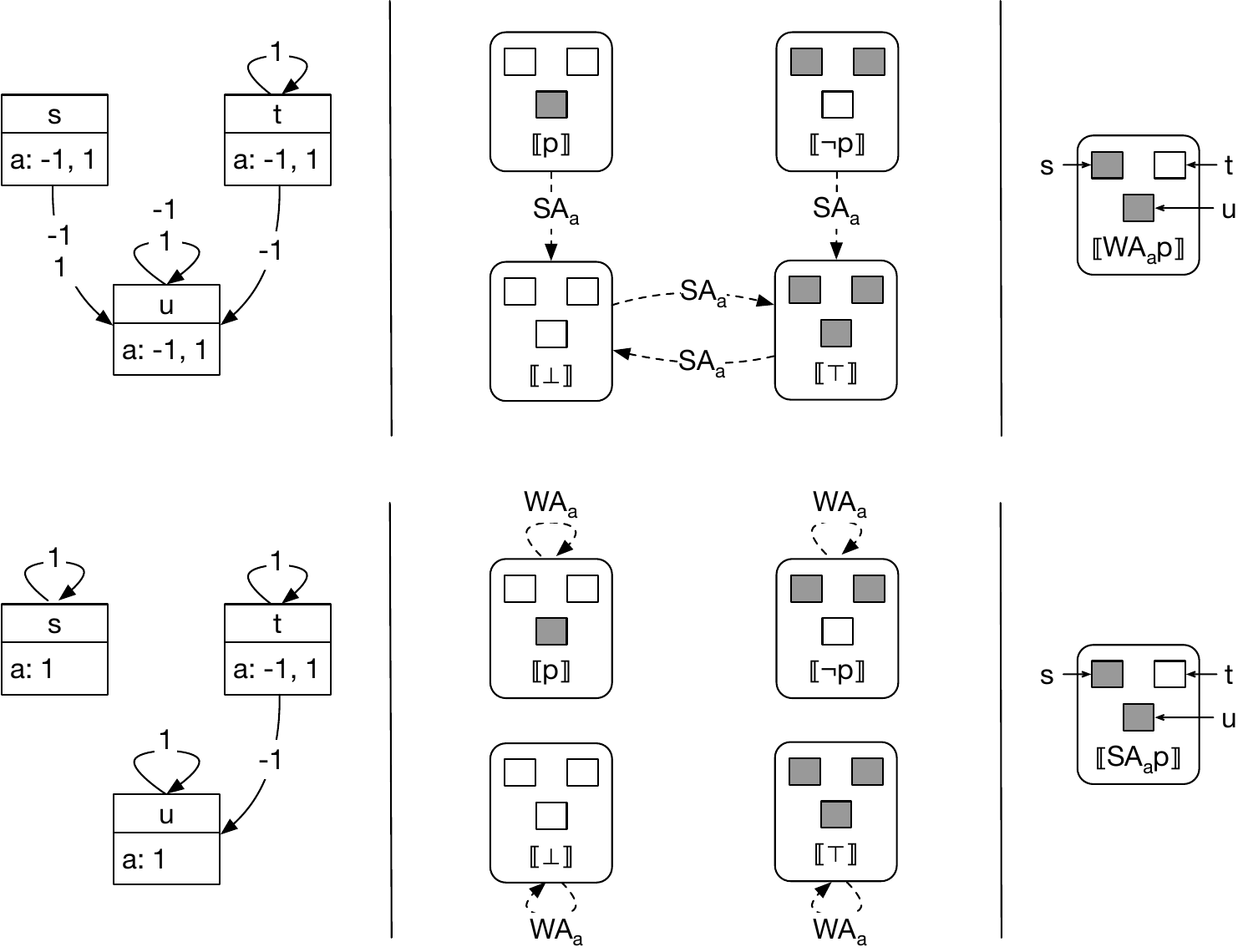}}
\caption{Toward mutual undefinability of $\WA$ (top) and $\SA$ (bottom) in single-agent deterministic settings.}
\label{fig:undef single}
\end{center}
\end{figure*}

\subsection{Undefinability in Single-Agent Deterministic Settings}\label{app:undefinability single}
To show the mutual undefinability of modalities $\WA$ and $\SA$ in single-agent deterministic settings, we use the two transition systems in the left of Figure~\ref{fig:undef single}. The proofs are omitted because they are similar to the proof of Theorem~\ref{th:undefinability WA} in Subsection~\ref{app:undefinability WA}.


\section{Proofs Toward Section~\ref{sec:axiom}}

\subsection{$\WE$ is ``Weak'', $\SE$ is ``Strong''}\label{app:WE weak ir} \label{app:SE strong ir}

\begin{theorem}\label{th:WE weak}
The rule
$
\dfrac{\phi\to\psi}{\WE_a\phi\to\WE_a\psi}
$    
is admissible.
\end{theorem}
\begin{proof}
Suppose that 
$\vdash \phi\to\psi$.
Then, 
$\vdash\phi\wedge\neg\psi\to\bot$
by the laws of propositional reasoning.
Thus, by inference rule~\ref{ir: WA},
$\vdash\WA_a(\phi\wedge\neg\psi)\to\WA_a\bot$.
Then, by the law of contraposition,
$\vdash\neg\WA_a\bot\to\neg\WA_a(\phi\wedge\neg\psi)$.
Hence, by axiom~\ref{ax: WA bot} and the Modus Ponens rule, 
$\vdash\neg\WA_a(\phi\wedge\neg\psi)$.
Thus,
$\vdash\neg\WE_a\phi\vee\WE_a\psi$
by the contrapositive of axiom~\ref{ax: WA WE}.
Therefore, $\vdash\WE_a\phi\to\WE_a\psi$ by propositional reasoning.
\end{proof}

\vspace{1mm}
Theorem~\ref{th:WE weak} shows that $\WE$ is a form of \textbf{weak} permission.

\begin{theorem}\label{th:SE strong}
The rule
$
\dfrac{\phi\to\psi}{\SE_a\psi\to\SE_a\phi}
$    
is admissible.
\end{theorem}
\begin{proof}
Suppose that 
$\vdash \phi\to\psi$.
Then, 
$\vdash\phi\wedge\neg\psi\to\bot$
by the laws of propositional reasoning.
Thus, by inference rule~\ref{ir: SA},
$\vdash\SA_a\bot\to\SA_a(\phi\wedge\neg\psi)$.
Hence, by axiom~\ref{ax: SA bot} and the Modus Ponens rule, 
$\vdash\SA_a(\phi\wedge\neg\psi)$.
Thus,
$\vdash\SE_a\phi\vee\neg\SE_a\psi$
by the contrapositive of axiom~\ref{ax: SE SA}.
Therefore, by propositional reasoning, $\vdash\SE_a\psi\to\SE_a\phi$.
\end{proof}

Theorem~\ref{th:SE strong} shows that $\SE$ is a form of \textbf{strong} permission.

\subsection{Proof of Lemma~\ref{lm:deduction}}\label{app:proof lm deduction}

\textbf{Lemma~\ref{lm:deduction} (deduction)} \textit{If $X,\phi\vdash \psi$, then $X\vdash\phi\to\psi$.}
\begin{proof}
Suppose that sequence $\psi_1,\dots,\psi_n$ is a proof of $\psi$ from set $X\cup\{\phi\}$ and the theorems of our logical system that uses the Modus Ponens rule {\em only}. In other words, for each $k\le n$, either
\begin{enumerate}
    \item $\vdash\psi_k$, or
    \item $\psi_k\in X$, or
    \item $\psi_k$ is equal to $\phi$, or
    \item there are $i,j<k$ such that formula $\psi_j$ has the form of $\psi_i\to\psi_k$.
\end{enumerate}
It suffices to show that $X\vdash\phi\to\psi_k$ for each $k\le n$. We prove this by induction on $k$ by considering the four cases above separately.

\vspace{1mm}
\noindent{\em Case 1}: $\vdash\psi_k$. Note that $\psi_k\to(\phi\to\psi_k)$ is a propositional tautology, and thus, is an axiom of our logical system. Hence, $\vdash\phi\to\psi_k$ by the Modus Ponens rule. Therefore, $X\vdash\phi\to\psi_k$. 

\vspace{1mm}
\noindent{\em Case 2}: $\psi_k\in X$. Then, $X\vdash\psi_k$.

\vspace{1mm}
\noindent{\em Case 3}: Formula $\psi_k$ is equal to $\phi$. Thus, $\phi\to\psi_k$ is a propositional tautology. Hence, $X\vdash\phi\to\psi_k$. 

\vspace{1mm}
\noindent{\em Case 4}: Formula $\psi_j$ is equal to $\psi_i\to\psi_k$ for some $i,j<k$. In this case, by the induction hypothesis, $X\vdash\phi\to\psi_i$ and $X\vdash\phi\to(\psi_i\to\psi_k)$. Note that formula
\begin{equation}\notag
(\phi\to\psi_i)\to\big((\phi\to(\psi_i\to\psi_k))\to(\phi\to\psi_k)\big)
\end{equation}
is a propositional tautology. Therefore, $X\vdash \phi\to\psi_k$ by applying the Modus Ponens rule twice.
\end{proof}

\subsection{Proofs of Lemma~\ref{lm:axiom deduction WE WA} and Lemma~\ref{lm:axiom deduction SE SA}} \label{app:proof of axioms deduction lemma}

\paragraph{Lemma~\ref{lm:axiom deduction WE WA}. $\vdash\WE_a\phi\wedge\neg\WA_a\psi\to\WE_a(\phi\wedge\neg\psi)$.}
\begin{proof}
The propositional tautology
$\phi\wedge\neg(\phi\wedge\neg\psi)\to\psi$, by rule~\ref{ir: WA}, implies
\begin{equation}\notag
\vdash \WA_a(\phi\wedge\neg(\phi\wedge\neg\psi))\to\WA_a\psi.
\end{equation}
Then, by the instance
\begin{equation}\notag
\WE_a\phi\wedge\neg\WE_a(\phi\wedge\neg\psi)\to\WA_a(\phi\wedge\neg(\phi\wedge\neg\psi))
\end{equation}
of axiom~\ref{ax: WA WE} and propositional reasoning
\begin{equation}\notag
\vdash\WE_a\phi\wedge\neg\WE_a(\phi\wedge\neg\psi)\to\WA_a\psi.
\end{equation}
Therefore, $\vdash\WE_a\phi\wedge\neg\WA_a\psi\to\WE_a(\phi\wedge\neg\psi)$ again by propositional reasoning.
\end{proof}

\paragraph{Lemma~\ref{lm:axiom deduction SE SA}. $\vdash\neg\SE_a\phi\wedge\SA_a\psi\to\neg\SE_a(\phi\wedge\neg\psi)$.}

\begin{proof}
The propositional tautology
$\phi\wedge\neg(\phi\wedge\neg\psi)\to\psi$, by rule~\ref{ir: SA}, implies
\begin{equation}\notag
\vdash \SA_a\psi\to\SA_a(\phi\wedge\neg(\phi\wedge\neg\psi)).
\end{equation}
By contraposition, the above statement implies
\begin{equation}\notag
\vdash \neg\SA_a(\phi\wedge\neg(\phi\wedge\neg\psi))\to\neg\SA_a\psi.
\end{equation}
Then, by the instance
\begin{equation}\notag
\neg\SE_a\phi\wedge\SE_a(\phi\wedge\neg\psi)\to\neg\SA_a(\phi\wedge\neg(\phi\wedge\neg\psi))
\end{equation}
of axiom~\ref{ax: SE SA} and propositional reasoning,
\begin{equation}\notag
\vdash\neg\SE_a\phi\wedge\SE_a(\phi\wedge\neg\psi)\to\neg\SA_a\psi.
\end{equation}
Therefore, $\vdash\neg\SE_a\phi\wedge\SA_a\psi\to\neg\SE_a(\phi\wedge\neg\psi)$ again by propositional reasoning.
\end{proof}

\section{Proof Toward Subsection~\ref{sec:completeness theorem}}
\label{app:proof of truth lemma}

To improve the readability, in Subsection~\ref{app:canonical model lemmas}, we first show four auxiliary lemmas used in the induction steps of the proof for Lemma~\ref{lm:truth lemma}. Then, in Subsection~\ref{app:core prove of truth lemma}, we prove Lemma~\ref{lm:truth lemma}.
After that, we prove Theorem~\ref{th:completeness} in Subsection~\ref{app:completeness theorem}.

\subsection{Auxiliary Lemmas}\label{app:canonical model lemmas}

Let us first see a theorem of the logic system that will be used in the proofs of the four auxiliary lemmas.

\begin{lemma}\label{lm:axiom deduction WA SA}
$\vdash\neg\WA_a\phi\wedge\SA_a\top\to\neg\WA_b\phi\wedge\SA_b\phi$ .
\end{lemma}

\begin{proof}
The propositional tautology $\phi\to\phi\wedge\top$, by rule~\ref{ir: WA} and propositional reasoning, implies
\begin{equation}\notag
\vdash\neg\WA_b(\phi\wedge\top)\to\neg\WA_b\phi,
\end{equation}
and by rule~\ref{ir: SA}, implies
\begin{equation}\notag
\vdash\SA_b(\phi\wedge\top)\to\SA_b\phi.
\end{equation}
Then, by propositional reasoning,
\begin{equation}\notag
\vdash\neg\WA_b(\phi\wedge\top)\wedge\SA_b(\phi\wedge\top)\to\neg\WA_b\phi\wedge\SA_b\phi.
\end{equation}
Thus, by the instance
\begin{equation}\notag
\neg\WA_a\phi\wedge\SA_a\top\to\neg\WA_b(\phi\wedge\top)\wedge\SA_b(\phi\wedge\top),
\end{equation}
of axiom~\ref{ax: WA SA} and propositional reasoning,
\begin{equation}\notag
\vdash\neg\WA_a\phi\wedge\SA_a\top\to\neg\WA_b\phi\wedge\SA_b\phi.
\end{equation}
\end{proof}

Now, we prove four auxiliary properties of the canonical model as four separated lemmas. 

\begin{lemma}\label{lm:auxiliar WA}
For any state $s\in S$ and any formula $\WA_a\phi\in s$, there is a tuple $(\delta,t)\in M_s$ such that $\delta_a\in D^s_a$ and $\phi\in t$.
\end{lemma}

\begin{proof}
Consider the sets of agents
\begin{align}
    B=& \{b\in\mathcal{A}\;|\;b\neq a, \SA_b\top\in s\},\label{eq:28-june-e}\\
    C=& \{c\in\mathcal{A}\;|\;c\neq a, \neg\SA_c\top\in s\}\label{eq:28-june-f}.
\end{align}
and the set of formulae
\begin{equation}\label{28-june-l} \hspace{-2mm}
\begin{aligned}
X = & \{\phi\}\cup\{\neg\psi\;|\;\neg\WA_a\psi\in s\}\\
& \cup \{\neg\chi\;|\;\exists\, b\in B(\neg\WA_b\chi\in s)\}\\
& \cup \{\neg(\tau\wedge\sigma)\;|\;\exists\, c\in C(\neg\WA_c \tau\in s, \SA_c \sigma\in s)\}.
\end{aligned}
\end{equation}
\begin{claim}
Set $X$ is consistent.    
\end{claim}
\begin{proof-of-claim}
Suppose the opposite, then there are formulae
\begin{align}
\neg&\WA_a\psi_1,\dots,\neg\WA_a\psi_k\in s,\label{eq:28-june-b} \\   
\neg&\WA_{b_1}\chi_1,\dots,\neg\WA_{b_\ell}\chi_{\ell}\in s,\label{eq:28-june-c}   \\ 
\neg&\WA_{c_1}\tau_1,\dots,\neg\WA_{c_m}\tau_{m}\in s,\label{eq:28-june-g}\\    
&\SA_{c_1}\sigma_1,\dots,\SA_{c_m}\sigma_{m}\in s,\label{eq:28-june-h}
\end{align}
where
\begin{align}
&b_1,\dots,b_\ell\in B,\label{eq:28-june-d}\\
&c_1,\dots,c_m\in C,
\end{align}
such that
\begin{equation}\label{eq:28-june-k}
\phi\wedge \bigwedge_{i\le k}\neg\psi_i
\wedge \bigwedge_{i\le \ell}\neg\chi_i\wedge
\bigwedge_{i\le m}\neg(\tau_i\wedge\sigma_i)\vdash \bot.
\end{equation}
By Lemma~\ref{lm:axiom deduction WA SA} and propositional reasoning, statements~\eqref{eq:28-june-c}, \eqref{eq:28-june-d}, and \eqref{eq:28-june-e} imply that
\begin{equation}\label{eq:28-june-i}
s\vdash \neg\WA_{a}\chi_i \text{ for each $i\le \ell$}.    
\end{equation}
At the same time, by axiom~\ref{ax: WA SA} and propositional reasoning, statements~\eqref{eq:28-june-g} and \eqref{eq:28-june-h} imply that
\begin{equation}\label{eq:28-june-j}
s\vdash \neg\WA_{a}(\tau_i\wedge\sigma_i) \text{ for each $i\le m$}.    
\end{equation}
On the other hand, by Lemma~\ref{lm:deduction} and propositional reasoning, statement~\eqref{eq:28-june-k} implies that
\begin{equation}\notag
\vdash \phi\to \bigvee_{i\le k}\psi_i
\vee \bigvee_{i\le \ell}\chi_i\vee
\bigvee_{i\le m}(\tau_i\wedge\sigma_i).
\end{equation}
By rule \ref{ir: WA}, the above statement implies that
\begin{equation}\notag
\vdash \WA_a\phi\to \WA_a\Big(\bigvee_{i\le k}\psi_i
\vee \bigvee_{i\le \ell}\chi_i\vee
\bigvee_{i\le m}(\tau_i\wedge\sigma_i)\Big).
\end{equation}
Thus, by the assumption $\WA_a\phi\in s$ of the lemma and the Modus Ponens rule, 
\begin{equation}\label{eq:28-june-a}
s\vdash \WA_a\Big(\bigvee_{i\le k}\psi_i
\vee \bigvee_{i\le \ell}\chi_i\vee
\bigvee_{i\le m}(\tau_i\wedge\sigma_i)\Big).
\end{equation}
Note that in the special case when $k=\ell=m=0$, statement~\eqref{eq:28-june-a} has the form $s\vdash\WA_a\bot$. This implies inconsistency of set $s$ due to axiom~\ref{ax: WA bot}. Thus, without loss of generality, we can assume that at least one of the integers $k$, $\ell$, and $m$ is positive. Hence, by multiple applications of axiom~\ref{ax: WA disjunction} and the laws of propositional reasoning, statement~\eqref{eq:28-june-a} implies that
\begin{equation}\notag
s\vdash \bigvee_{i\le k}\WA_a\psi_i
\vee \bigvee_{i\le \ell}\WA_a\chi_i\vee
\bigvee_{i\le m}\WA_a(\tau_i\wedge\sigma_i).
\end{equation}
This statement contradicts statements~\eqref{eq:28-june-b}, \eqref{eq:28-june-i}, and \eqref{eq:28-june-j} because set $s$ is consistent.
\end{proof-of-claim}

Let $t$ be any maximal consistent extension of set $X$. Such $t$ exists by Lemma~\ref{lm:Lindenbaum}. Then, $t\in S$ by Definition~\ref{df:canonical S}.
\begin{claim}\label{cl:29-june-a}
For each agent $c\in C$, at least one of the following statements is true:
\begin{enumerate}
    \item $\neg\tau\in t$ for each formula $\neg\WA_c\tau\in s$, {\bf or}
    \item $\neg\sigma\in t$ for each formula $\SA_c\sigma\in s$.
\end{enumerate}
\end{claim}
\begin{proof-of-claim}
Suppose the opposite. Then, there are formulae  $\neg\WA_c\tau\in s$ and $\SA_c\sigma\in s$ such that  $\neg\tau\notin t$ and  $\neg\sigma\notin t$. Thus, $\tau\in t$ and  $\sigma\in t$ because $t$ is a maximal consistent set.
Hence, $t\vdash \tau\wedge\sigma$ by propositional reasoning. Then, $\neg ( \tau\wedge\sigma)\notin t$ because set $t$ is consistent. Therefore, $\neg ( \tau\wedge\sigma)\notin X$ because $X\subseteq t$, which contradicts line~3 of statement~\eqref{28-june-l}.
\end{proof-of-claim}

Note that the sets $\{a\}$, $B$, and $C$ form a partition of the set $\mathcal{A}$ of all agents due to statements~\eqref{eq:28-june-e} and \eqref{eq:28-june-f}. Consider any action profile $\delta$ that satisfies the following conditions:
\begin{enumerate}
    \item $\delta_x=\top^+$ for each $x\in \{a\}\cup B$;
    \item $\delta_c\in\{\top^+,\top^-\}$ for each $c\in C$ such that 
    \begin{enumerate}
        \item if $\delta_c=\top^+$, then $\neg\tau\in t$ for each $\neg\WA_c\tau\in s$;
        \item if $\delta_c=\top^-$, then $\neg\sigma\in t$ for each $\SA_c\sigma\in s$.
    \end{enumerate}
\end{enumerate}
The existence of at least one such action profile $\delta$ follows from Definition~\ref{df:canonical action space & D} and Claim~\ref{cl:29-june-a}.

\begin{claim}
$(\delta,t)\in M_s$.
\end{claim}
\begin{proof-of-claim}
It suffices to verify that conditions 1-4 of Definition~\ref{df:canonical M} are satisfied for the tuple $(\delta,t)$ for each agent $x\in\mathcal{A}$. Recall that the sets $\{a\},B,C$ form a partition of the set $\mathcal{A}$ due to statements~\eqref{eq:28-june-e} and \eqref{eq:28-june-f}. Thus, it suffices to consider the following three cases:

\vspace{0.5mm}\noindent{\em Case 1}: $x\in \{a\}\cup B$. 
Then, $\delta_x=\top^+$ by the choice of profile $\delta$.
Condition 1 of Definition~\ref{df:canonical M} follows from lines~1 and 2 of statement~\eqref{28-june-l} and  $X\subseteq t$.  Conditions 2 and 3 are satisfied because $\top\in t$. Condition 4 is satisfied because $\delta_x=\top^+\in D^s_x$. 

\vspace{0.5mm}\noindent{\em Case 2}: $x\in C$ and $\delta_x=\top^+$. Condition 1 follows from item 2(a) of the choice of profile $\delta$. Conditions 2, 3, and 4 are similar to the previous case.

\vspace{0.5mm}\noindent{\em Case 3}: $x\in C$ and $\delta_x=\top^-$. Condition 1 is satisfied because $\delta_x=\top^-\notin D^s_x$. Conditions 2 and 3 are satisfied because $\top\in t$. Condition 4 follows from item 2(b) of the choice of profile $\delta$. 
\end{proof-of-claim}

To finish the proof of the lemma, note that $\delta_a=\top^+\in D^s_a$ by the choice of profile $\delta$ and Definition~\ref{df:canonical action space & D}. Also, observe that $\phi\in X\subseteq t$ by line~1 of statement~\eqref{28-june-l} and the formation of set $t$.
\end{proof}

\begin{lemma}\label{lm:auxiliar WE}
For any state $s\in S$, any formula $\neg\WE_a\phi\in s$, and any action $i\in D_a^s$, there is a tuple $(\delta,t)\in M_s$ such that $\delta_a=i$ and $\neg\phi\in t$.
\end{lemma}

\begin{proof}
By the assumption $i\in D_a^s$ of the lemma and Definition~\ref{df:canonical action space & D},
\begin{equation}\label{eq:7-July-11-part1}
   i=\epsilon^+ 
\end{equation}
for some $\epsilon\in\Phi$. Let
\begin{equation}\label{eq:7-July-11-part2}
\widehat{\epsilon}=
\begin{cases}
    \epsilon, &\text{if } \WE_a\epsilon\in s;\\
    \top, &\text{otherwise.}
\end{cases}
\end{equation}
Note that, by axiom~\ref{ax:WE T}, statement~\eqref{eq:7-July-11-part2} implies that
\begin{equation}\label{eq:7-July-14}
\WE_a\widehat{\epsilon}\in s.
\end{equation}

\noindent Consider the sets of agents
\begin{align}
    B=& \{b\in\mathcal{A}\;|\;b\neq a, \SA_b\top\in s\},\label{eq:7-July-2}\\
    C=& \{c\in\mathcal{A}\;|\;c\neq a, \neg\SA_c\top\in s\}\label{eq:9-July-1},
\end{align}
and the set of formulae
\begin{equation}\label{eq:9-July-2} \hspace{-2mm}
\begin{aligned}
X =
&\{\neg\phi, \widehat{\epsilon}\;\}\cup\{\neg\psi\;|\;\neg\WA_a\psi\in s\}\\
&\cup \{\neg\chi\;|\;\exists\, b\in B(\neg\WA_b\chi\in s)\}\\
&\cup \{\neg(\tau\wedge\sigma)\;|\;\exists\, c\in C(\neg\WA_c \tau\in s, \SA_c \sigma\in s)\}.
\end{aligned}
\end{equation}
\begin{claim}
Set $X$ is consistent.    
\end{claim}
\begin{proof-of-claim}
Suppose the opposite, then there are formulae
\begin{align}
\neg&\WA_a\psi_1,\dots,\neg\WA_a\psi_k\in s, \label{eq:7-July-6}\\   
\neg&\WA_{b_1}\chi_1,\dots,\neg\WA_{b_\ell}\chi_{\ell}\in s, \label{eq:7-July-1}\\ 
\neg&\WA_{c_1}\tau_1,\dots,\neg\WA_{c_m}\tau_{m}\in s,\label{eq:7-July-4}\\    
&\SA_{c_1}\sigma_1,\dots,\SA_{c_m}\sigma_{m}\in s, \label{eq:7-July-5}
\end{align}
where
\begin{align}
    &b_1,\dots,b_\ell\in B,\label{eq:7-July-3}\\
    &c_1,\dots,c_m\in C, 
\end{align}
such that
\begin{equation}\label{eq:7-July-9}
\neg\phi\wedge\widehat{\epsilon}\wedge
\bigwedge_{i\leq k}\neg\psi_i\wedge
\bigwedge_{i\leq\ell}\neg\chi_i\wedge
\bigwedge_{i\leq m}\neg(\tau_i\wedge\sigma_i)\vdash \bot.
\end{equation}
Note that, by axiom~\ref{ax: WA WE}, statement~\eqref{eq:7-July-14} and the assumption $\neg\WE_a\phi\in s$ of the lemma imply that
\begin{equation}\label{eq:7-July-12}
s\vdash \WA_a(\,\widehat{\epsilon}\,\wedge\neg\phi).
\end{equation}
By Lemma~\ref{lm:axiom deduction WA SA} and propositional reasoning, statements~\eqref{eq:7-July-1}, \eqref{eq:7-July-3}, and \eqref{eq:7-July-2} imply that
\begin{equation}\label{eq:7-July-7}
s\vdash \neg\WA_{a}\chi_i \text{ for each }i\le \ell.    
\end{equation}
At the same time, by axiom~\ref{ax: WA SA} and propositional reasoning, statements~\eqref{eq:7-July-4} and \eqref{eq:7-July-5} imply that
\begin{equation}\label{eq:7-July-8}
s\vdash \neg\WA_{a}(\tau_i\wedge\sigma_i) \text{ for each }i\le m.    
\end{equation}
By applying the contrapositive of axiom~\ref{ax: WA disjunction} multiple times, statements~\eqref{eq:7-July-6}, \eqref{eq:7-July-7} and \eqref{eq:7-July-8} imply that
\begin{equation}\label{eq:7-July-13}
s\vdash \neg\WA_a\Big(\bigvee_{i\le k}\psi_i
\vee \bigvee_{i\le \ell}\chi_i\vee
\bigvee_{i\le m}(\tau_i\wedge\sigma_i)\Big).
\end{equation}
In the special case where $k=\ell=m=0$, statement~\eqref{eq:7-July-13} follows directly from axiom~\ref{ax: WA bot}.

On the other hand, by Lemma~\ref{lm:deduction} and propositional reasoning, statement~\eqref{eq:7-July-9} implies that
\begin{equation}\notag
\vdash (\,\widehat{\epsilon}\,\wedge\neg\phi)\to
\Big(\bigvee_{i\le k}\psi_i
\vee \bigvee_{i\le \ell}\chi_i\vee
\bigvee_{i\le m}(\tau_i\wedge\sigma_i)\Big).
\end{equation}
By rule~\ref{ir: WA}, the above statement implies that
\begin{equation}\notag
\vdash \WA_a(\,\widehat{\epsilon}\,\wedge\neg\phi)\to
\WA_a\Big(\bigvee_{i\le k}\psi_i
\vee \bigvee_{i\le \ell}\chi_i\vee
\bigvee_{i\le m}(\tau_i\wedge\sigma_i)\Big).
\end{equation}
Together with statement~\eqref{eq:7-July-12} and the Modus Ponens rule, the above statement implies that
\begin{equation}\notag
s\vdash \WA_a\Big(\bigvee_{i\le k}\psi_i
\vee \bigvee_{i\le \ell}\chi_i\vee
\bigvee_{i\le m}(\tau_i\wedge\sigma_i)\Big),
\end{equation}
which contradicts statement~\eqref{eq:7-July-13} because $s$ is consistent.
\end{proof-of-claim}

Let $t$ be any maximal consistent extension of set $X$. Such $t$ exists by Lemma~\ref{lm:Lindenbaum}. Then, $t\in S$ by Definition~\ref{df:canonical S}.

\begin{claim}\label{cl:9-July-1}
For each agent $c\in C$, at least one of the following statements is true:
\begin{enumerate}
    \item $\neg\tau\in t$ for each formula $\neg\WA_c\tau\in s$, {\bf or}
    \item $\neg\sigma\in t$ for each formula $\SA_c\sigma\in s$.
\end{enumerate}
\end{claim}
\begin{proof-of-claim}
Suppose the opposite. Then, there are formulae  $\neg\WA_c\tau\in s$ and $\SA_c\sigma\in s$ such that  $\neg\tau\notin t$ and  $\neg\sigma\notin t$. Thus, $\tau\in t$ and  $\sigma\in t$ because $t$ is a maximal consistent set.
Hence, $t\vdash \tau\wedge\sigma$ by propositional reasoning. Then, $\neg ( \tau\wedge\sigma)\notin t$ because set $t$ is consistent. Therefore, $\neg ( \tau\wedge\sigma)\notin X$ because $X\subseteq t$, which contradicts line~3 of statement~\eqref{eq:9-July-2}.
\end{proof-of-claim}

Note that the sets $\{a\}$, $B$, and $C$ form a partition of the set $\mathcal{A}$ of all agents due to statements~\eqref{eq:7-July-2} and \eqref{eq:9-July-1}. Consider any action profile $\delta$ that satisfies the following conditions:
\begin{enumerate}
    \item $\delta_a=\epsilon^+$;
    \item $\delta_b=\top^+$ for each $b\in B$;
    \item $\delta_c\in\{\top^+,\top^-\}$ for each $c\in C$ such that 
    \begin{enumerate}
        \item if $\delta_c=\top^+$, then $\neg\tau\in t$ for each $\neg\WA_c\tau\in s$;
        \item if $\delta_c=\top^-$, then $\neg\sigma\in t$ for each $\SA_c\sigma\in s$.
    \end{enumerate}
\end{enumerate}
The existence of at least one such action profile $\delta$ follows from Definition~\ref{df:canonical action space & D} and Claim~\ref{cl:9-July-1}.

\begin{claim}
$(\delta,t)\in M_s$.
\end{claim}
\begin{proof-of-claim}
It suffices to verify that conditions 1-4 of Definition~\ref{df:canonical M} are satisfied for the tuple $(\delta,t)$ for each agent $x\in\mathcal{A}$. Recall that the sets $\{a\},B,C$ form a partition of the set $\mathcal{A}$  due to statements~\eqref{eq:7-July-2} and \eqref{eq:9-July-1}. Thus, by the choice of profile $\delta$, it suffices to consider the following four cases:

\vspace{0.5mm}\noindent{\em Case 1}: $x=a$. Then, $\delta_x=\epsilon^+$ by the choice of profile $\delta$. Condition~1 of Definition~\ref{df:canonical M} follows from line~1 of statement~\eqref{eq:9-July-2} and $X\subseteq t$. Condition~2 is satisfied because of statement~\eqref{eq:7-July-11-part2}, line~1 of statement~\eqref{eq:9-July-2}, and $X\subseteq t$. Conditions~3 and 4 are satisfied because $\delta_x=\epsilon^+\in D_x^s$.

\vspace{0.5mm}\noindent{\em Case 2}: $x\in B$. 
Then, $\delta_x=\top^+$ by the choice of $\delta$.
Condition 1 of Definition~\ref{df:canonical M} is satisfied by line~2 of statement~\eqref{eq:9-July-2} and $X\subseteq t$.  Conditions 2 and 3 are satisfied because $\top\in t$. Condition 4 is satisfied because $\delta_x=\top^+\in D^s_x$. 

\vspace{0.5mm}\noindent{\em Case 3}: $x\in C$ and $\delta_x=\top^+$. Condition 1 follows from item 3(a) of the choice of profile $\delta$. Conditions 2, 3, and 4 are similar to the previous case.

\vspace{0.5mm}\noindent{\em Case 4}: $x\in C$ and $\delta_x=\top^-$. Conditions 1 and 2 are satisfied because $\delta_x=\top^-\notin D^s_x$. Condition 3 is satisfied because $\top\in t$. Condition 4 follows from item 3(b) of the choice of profile $\delta$. 
\end{proof-of-claim}

To finish the proof of this lemma, note that $\delta_a=\epsilon^+=i$ by the choice of profile $\delta$ and statement~\eqref{eq:7-July-11-part1}. At the same time, $\neg\phi\in X\subseteq t$ by line~1 of statement~\eqref{eq:9-July-2} and the formation of set $t$.
\end{proof}

\begin{lemma}\label{lm:auxiliar SE}
For any state $s\in S$, any formula $\SE_a\phi\in s$, and any action $i\in \Delta_a^s\setminus D_a^s$, there is a tuple $(\delta,t)\in M_s$ such that $\delta_a=i$ and $\neg\phi\in t$.
\end{lemma}

\begin{proof}
By the assumption $i\in\Delta_a^s\setminus D_a^s$ of the lemma and Definition~\ref{df:canonical action space & D},
\begin{equation}\label{eq:9-July-18}
i=\epsilon^-
\end{equation}
for some $\epsilon\in\Phi$. Let
\begin{equation}\label{eq:9-July-3}
\widehat{\epsilon}=
\begin{cases}
    \epsilon, &\text{if } \neg\SE_a\epsilon\in s;\\
    \top, &\text{otherwise.}
\end{cases}
\end{equation}
Note that the assumptions $\delta_a=i\in\Delta_a^s\setminus D_a^s$ of the lemma imply that $\SA_a\top\notin s$ by Definition~\ref{df:canonical action space & D}. Since $s$ is a maximal consistent set, $\neg\SA_a\top\in s$. Then, $\neg\SE_a\top\in s$ by the contrapositive of axiom~\ref{ax: SA SE T}. Thus, by statement~\eqref{eq:9-July-3},
\begin{equation}\label{eq:9-July-14}
\neg\SE_a\widehat{\epsilon}\in s.
\end{equation}

\noindent Consider the sets of agents
\begin{align}
    B=& \{b\in\mathcal{A}\;|\;b\neq a, \SA_b\top\in s\},\label{eq:9-July-4}\\
    C=& \{c\in\mathcal{A}\;|\;c\neq a, \neg\SA_c\top\in s\},\label{eq:9-July-5}
\end{align}
and the set of formulae
\begin{equation}\label{eq:9-July-6}
\hspace{-2mm}
\begin{aligned}
X =
&\{\neg\phi, \widehat{\epsilon}\;\}\cup\{\neg\psi\;|\;\SA_a\psi\in s\}\\
&\cup \{\neg\chi\;|\;\exists\, b\in B(\neg\WA_b\chi\in s)\}\\
&\cup \{\neg(\tau\wedge\sigma)\;|\;\exists\, c\in C(\neg\WA_c \tau\in s, \SA_c \sigma\in s)\}.
\end{aligned}  
\end{equation}
\begin{claim}
Set $X$ is consistent.    
\end{claim}
\begin{proof-of-claim}
Suppose the opposite, then there are formulae
\begin{align}
&\SA_a\psi_1,\dots,\SA_a\psi_k\in s, \label{eq:9-July-7}\\   
\neg&\WA_{b_1}\chi_1,\dots,\neg\WA_{b_\ell}\chi_{\ell}\in s, \label{eq:9-July-8}\\ 
\neg&\WA_{c_1}\tau_1,\dots,\neg\WA_{c_m}\tau_{m}\in s,\label{eq:9-July-9}\\    
&\SA_{c_1}\sigma_1,\dots,\SA_{c_m}\sigma_{m}\in s, \label{eq:9-July-10}
\end{align}
where
\begin{align}
    &b_1,\dots,b_\ell\in B,\label{eq:9-July-11}\\
    &c_1,\dots,c_m\in C, 
\end{align}
such that
\begin{equation}\label{eq:9-July-12}
\neg\phi\wedge\widehat{\epsilon}\wedge
\bigwedge_{i\leq k}\neg\psi_i\wedge
\bigwedge_{i\leq\ell}\neg\chi_i\wedge
\bigwedge_{i\leq m}\neg(\tau_i\wedge\sigma_i)\vdash \bot.
\end{equation}
Note that, by axiom~\ref{ax: SE SA}, statement~\eqref{eq:9-July-14} and the assumption $\SE_a\phi\in s$ of the lemma imply that
\begin{equation}\label{eq:9-July-13}
s\vdash \neg\SA_a(\,\widehat{\epsilon}\,\wedge\neg\phi).
\end{equation}
By Lemma~\ref{lm:axiom deduction WA SA} and propositional reasoning, statements~\eqref{eq:9-July-8}, \eqref{eq:9-July-11}, and \eqref{eq:9-July-4} imply that
\begin{equation}\label{eq:9-July-15}
s\vdash \SA_{a}\chi_i \text{ for each }i\le \ell.    
\end{equation}
At the same time, by axiom~\ref{ax: WA SA} and propositional reasoning, statements~\eqref{eq:9-July-9} and \eqref{eq:9-July-10} imply that
\begin{equation}\label{eq:9-July-16}
s\vdash \SA_{a}(\tau_i\wedge\sigma_i) \text{ for each }i\le m.    
\end{equation}
By applying axiom~\ref{ax: SA conjunction} multiple times, statements~\eqref{eq:9-July-7}, \eqref{eq:9-July-15} and \eqref{eq:9-July-16} imply that
\begin{equation}\label{eq:9-July-17}
s\vdash \SA_a\Big(\bigvee_{i\le k}\psi_i
\vee \bigvee_{i\le \ell}\chi_i\vee
\bigvee_{i\le m}(\tau_i\wedge\sigma_i)\Big).
\end{equation}
In the special case where $k=\ell=m=0$, statement~\eqref{eq:9-July-17} follows directly from axiom~\ref{ax: SA bot}.

On the other hand, by Lemma~\ref{lm:deduction} and propositional reasoning, statement~\eqref{eq:9-July-12} implies that
\begin{equation}\notag
\vdash (\,\widehat{\epsilon}\,\wedge\neg\phi)\to
\Big(\bigvee_{i\le k}\psi_i
\vee \bigvee_{i\le \ell}\chi_i\vee
\bigvee_{i\le m}(\tau_i\wedge\sigma_i)\Big).
\end{equation}
By rule~\ref{ir: SA}, the above statement implies
\begin{equation}\notag
\vdash
\SA_a\Big(\bigvee_{i\le k}\psi_i
\vee \bigvee_{i\le \ell}\chi_i\vee
\bigvee_{i\le m}(\tau_i\wedge\sigma_i)\Big) \to \SA_a(\,\widehat{\epsilon}\,\wedge\neg\phi).
\end{equation}
Together with statement~\eqref{eq:9-July-17} and the Modus Ponens rule, the above statement implies that
\begin{equation}\notag
s\vdash\SA_a(\,\widehat{\epsilon}\,\wedge\neg\phi),
\end{equation}
which contradicts statement~\eqref{eq:9-July-13} because $s$ is consistent.
\end{proof-of-claim}

Let $t$ be any maximal consistent extension of set $X$. Such $t$ exists by Lemma~\ref{lm:Lindenbaum}. Then, $t\in S$ by Definition~\ref{df:canonical S}.

\begin{claim}\label{cl:9-July-2}
For each agent $c\in C$, at least one of the following statements is true:
\begin{enumerate}
    \item $\neg\tau\in t$ for each formula $\neg\WA_c\tau\in s$, {\bf or}
    \item $\neg\sigma\in t$ for each formula $\SA_c\sigma\in s$.
\end{enumerate}
\end{claim}
\begin{proof-of-claim}
Suppose the opposite. Then, there are formulae  $\neg\WA_c\tau\in s$ and $\SA_c\sigma\in s$ such that  $\neg\tau\notin t$ and  $\neg\sigma\notin t$. Thus, $\tau\in t$ and  $\sigma\in t$ because $t$ is a maximal consistent set.
Hence, $t\vdash \tau\wedge\sigma$ by propositional reasoning. Then, $\neg ( \tau\wedge\sigma)\notin t$ because set $t$ is consistent. Therefore, $\neg ( \tau\wedge\sigma)\notin X$ because $X\subseteq t$, which contradicts line~3 of statement~\eqref{eq:9-July-6}.
\end{proof-of-claim}

Note that the sets $\{a\}$, $B$, and $C$ form a partition of the set $\mathcal{A}$ of all agents due to statements~\eqref{eq:9-July-4} and \eqref{eq:9-July-5}. Consider any action profile $\delta$ that satisfies the following conditions:
\begin{enumerate}
    \item $\delta_a=\epsilon^-$;
    \item $\delta_b=\top^+$ for each $b\in B$;
    \item $\delta_c\in\{\top^+,\top^-\}$ for each $c\in C$ such that 
    \begin{enumerate}
        \item if $\delta_c=\top^+$, then $\neg\tau\in t$ for each $\neg\WA_c\tau\in s$;
        \item if $\delta_c=\top^-$, then $\neg\sigma\in t$ for each $\SA_c\sigma\in s$.
    \end{enumerate}
\end{enumerate}
The existence of at least one such action profile $\delta$ follows from Definition~\ref{df:canonical action space & D} and Claim~\ref{cl:9-July-2}.

\begin{claim}
$(\delta,t)\in M_s$.
\end{claim}
\begin{proof-of-claim}
It suffices to verify that conditions 1-4 of Definition~\ref{df:canonical M} are satisfied for the tuple $(\delta,t)$ for each agent $x\in\mathcal{A}$. Recall that the sets $\{a\},B,C$ form a partition of the set $\mathcal{A}$ due to statements~\eqref{eq:9-July-4} and \eqref{eq:9-July-5}. Thus, by the choice of profile $\delta$, it suffices to consider the following four cases:

\vspace{0.5mm}\noindent{\em Case 1}: $x=a$. Then, $\delta_x=\epsilon^-$ by the choice of $\delta$. Conditions~1 and 2 of Definition~\ref{df:canonical M} follows from $\delta_x=\epsilon^-\notin D_x^s$. Condition~3 is satisfied because of statement~\eqref{eq:9-July-3}, line~1 of statement~\eqref{eq:9-July-6}, and $X\subseteq t$. Condition~4 follows from line~1 of statement~\eqref{eq:9-July-6} and $X\subseteq t$.

\vspace{0.5mm}\noindent{\em Case 2}: $x\in B$. 
Then, $\delta_x=\top^+$ by the choice of $\delta$.
Condition~1 of Definition~\ref{df:canonical M} follows from line~2 of statement~\eqref{eq:9-July-6} and $X\subseteq t$.  Conditions 2 and 3 are satisfied because $\top\in t$. Condition 4 is satisfied because $\delta_x=\top^+\notin\Delta_x^s\setminus D^s_x$. 

\vspace{0.5mm}\noindent{\em Case 3}: $x\in C$ and $\delta_x=\top^+$. Condition 1 follows from item 3(a) of the choice of profile $\delta$. Conditions 2, 3, and 4 are similar to the previous case.

\vspace{0.5mm}\noindent{\em Case 4}: $x\in C$ and $\delta_x=\top^-$. Conditions 1 and 2 are satisfied because $\delta_x=\top^-\notin D^s_x$. Condition 3 is satisfied because $\top\in t$. Condition 4 follows from item 3(b) of the choice of profile $\delta$. 
\end{proof-of-claim}

To finish the proof of this lemma, note that $\delta_a=\epsilon^-=i$ by the choice of profile $\delta$ and statement~\eqref{eq:9-July-18}. Also, note that $\neg\phi\in X\subseteq t$ by line~1 of statement~\eqref{eq:9-July-6} and the formation of set $t$.
\end{proof}

\begin{lemma}\label{lm:auxiliar SA}
\!For any state $s\!\in\! S$ and any formula $\neg\SA_a\phi\in s$, there is a tuple $(\delta,t)\in M_s$ such that $\delta_a\in\Delta_a^s\setminus D_a^s$ and $\phi\in t$.
\end{lemma}

\begin{proof}
Consider the sets of agents
\begin{align}
    B=& \{b\in\mathcal{A}\;|\;b\neq a, \SA_b\top\in s\},\label{eq:10-July-1}\\
    C=& \{c\in\mathcal{A}\;|\;c\neq a, \neg\SA_c\top\in s\},\label{eq:10-July-2}
\end{align}
and the set of formulae
\begin{equation}\label{eq:10-July-3} \hspace{-2mm}
\begin{aligned}
X = & \{\phi\}\cup\{\neg\psi\;|\;\SA_a\psi\in s\}\\
& \cup \{\neg\chi\;|\;\exists\, b\in B(\neg\WA_b\chi\in s)\}\\
& \cup \{\neg(\tau\wedge\sigma)\;|\;\exists\, c\in C(\neg\WA_c \tau\in s, \SA_c \sigma\in s)\}.
\end{aligned}
\end{equation}
\begin{claim}
Set $X$ is consistent.    
\end{claim}
\begin{proof-of-claim}
Suppose the opposite, then there are formulae
\begin{align}
&\SA_a\psi_1,\dots,\SA_a\psi_k\in s,\label{eq:10-July-4} \\   
\neg&\WA_{b_1}\chi_1,\dots,\neg\WA_{b_\ell}\chi_{\ell}\in s,\label{eq:10-July-5}   \\ 
\neg&\WA_{c_1}\tau_1,\dots,\neg\WA_{c_m}\tau_{m}\in s,\label{eq:10-July-6}\\    
&\SA_{c_1}\sigma_1,\dots,\SA_{c_m}\sigma_{m}\in s,\label{eq:10-July-7}
\end{align}
where
\begin{align}
&b_1,\dots,b_\ell\in B,\label{eq:10-July-8}\\
&c_1,\dots,c_m\in C,
\end{align}
such that
\begin{equation}\label{eq:10-July-9}
\phi\wedge \bigwedge_{i\le k}\neg\psi_i
\wedge \bigwedge_{i\le \ell}\neg\chi_i\wedge
\bigwedge_{i\le m}\neg(\tau_i\wedge\sigma_i)\vdash\bot.
\end{equation}
By Lemma~\ref{lm:axiom deduction WA SA} and propositional reasoning, statements~\eqref{eq:10-July-5}, \eqref{eq:10-July-8}, and \eqref{eq:10-July-1} imply that,
\begin{equation}\label{eq:10-July-10}
s\vdash \SA_{a}\chi_i \text{ for each $i\le \ell$}.    
\end{equation}
At the same time, by axiom~\ref{ax: WA SA} and propositional reasoning, statements~\eqref{eq:10-July-6} and \eqref{eq:10-July-7} imply that,
\begin{equation}\label{eq:10-July-11}
s\vdash \SA_{a}(\tau_i\wedge\sigma_i) \text{ for each $i\le m$}.    
\end{equation}
By applying axiom~\ref{ax: SA conjunction} multiple times, statements~\eqref{eq:10-July-4}, \eqref{eq:10-July-10} and \eqref{eq:10-July-11} imply that
\begin{equation}\label{eq:10-July-12}
s\vdash \SA_a\Big(\bigvee_{i\le k}\psi_i
\vee \bigvee_{i\le \ell}\chi_i\vee
\bigvee_{i\le m}(\tau_i\wedge\sigma_i)\Big).
\end{equation}
In the special case where $k=\ell=m=0$, statement~\eqref{eq:10-July-12} follows directly from axiom~\ref{ax: SA bot}.

On the other hand, by Lemma~\ref{lm:deduction} and propositional reasoning, statement~\eqref{eq:10-July-9} implies that
\begin{equation}\notag
\vdash \phi\to \bigvee_{i\le k}\psi_i
\vee \bigvee_{i\le \ell}\chi_i\vee
\bigvee_{i\le m}(\tau_i\wedge\sigma_i).
\end{equation}
By rule \ref{ir: SA}, this implies that
\begin{equation}\notag
\vdash \SA_a\Big(\bigvee_{i\le k}\psi_i
\vee \bigvee_{i\le \ell}\chi_i\vee
\bigvee_{i\le m}(\tau_i\wedge\sigma_i)\Big)\to\SA_a\phi.
\end{equation}
Together with statement~\eqref{eq:10-July-12} and the Modus Ponens rule, the above statement implies that
\begin{equation}\notag
s\vdash \SA_a\phi,
\end{equation}
which contradicts the assumption $\neg\SA_a\phi\in s$ of the lemma because set $s$ is consistent.
\end{proof-of-claim}

Let $t$ be any maximal consistent extension of set $X$. Such $t$ exists by Lemma~\ref{lm:Lindenbaum}. Then, $t\in S$ by Definition~\ref{df:canonical S}.
\begin{claim}\label{cl:10-July-1}
For each agent $c\in C$ {at least one} of the following statements is true:
\begin{enumerate}
    \item $\neg\tau\in t$ for each formula $\neg\WA_c\tau\in s$, {\bf or}
    \item $\neg\sigma\in t$ for each formula $\SA_c\sigma\in s$.
\end{enumerate}
\end{claim}
\begin{proof-of-claim}
Suppose the opposite. Then, there are formulae  $\neg\WA_c\tau\in s$ and $\SA_c\sigma\in s$ such that  $\neg\tau\notin t$ and  $\neg\sigma\notin t$. Thus, $\tau\in t$ and  $\sigma\in t$ because $t$ is a maximal consistent set.
Hence, $t\vdash \tau\wedge\sigma$ by propositional reasoning. Then, $\neg ( \tau\wedge\sigma)\notin t$ because set $t$ is consistent. Therefore, $\neg ( \tau\wedge\sigma)\notin X$ because $X\subseteq t$, which contradicts line~3 of statement~\eqref{eq:10-July-3}.
\end{proof-of-claim}

Note that the sets $\{a\}$, $B$, and $C$ form a partition of the set $\mathcal{A}$ of all agents due to statements~\eqref{eq:10-July-1} and \eqref{eq:10-July-2}. Consider any action profile $\delta$ that satisfies the following conditions:
\begin{enumerate}
    \item $\delta_a=\top^-$
    \item $\delta_b=\top^+$ for each $b\in B$;
    \item $\delta_c\in\{\top^+,\top^-\}$ for each $c\in C$ such that 
    \begin{enumerate}
        \item if $\delta_c=\top^+$, then $\neg\tau\in t$ for each $\neg\WA_c\tau\in s$;
        \item if $\delta_c=\top^-$, then $\neg\sigma\in t$ for each $\SA_c\sigma\in s$.
    \end{enumerate}
\end{enumerate}
The existence of at least one such action profile $\delta$ follows from Definition~\ref{df:canonical action space & D} and Claim~\ref{cl:10-July-1}.

\begin{claim}
$(\delta,t)\in M_s$.
\end{claim}
\begin{proof-of-claim}
It suffices to verify that conditions 1-4 of Definition~\ref{df:canonical M} are satisfied for the tuple $(\delta,t)$ for each agent $x\in\mathcal{A}$. Recall that the sets $\{a\},B,C$ form a partition of the set $\mathcal{A}$ due to statements~\eqref{eq:10-July-1} and \eqref{eq:10-July-2}. Thus, by the choice of profile $\delta$, it suffices to consider the following four cases:

\vspace{0.5mm}\noindent{\em Case 1}: $x=a$. Then, $\delta_x=\top^-$ by the choice of profile~$\delta$.
Conditions~1 and 2 of Definition~\ref{df:canonical M} are satisfied because $\delta_x=\top^-\notin D_x^s$.  Condition~3 is satisfied because $\top\in t$. Condition~4 follows from line~1 of statement~\eqref{eq:10-July-3} and $X\subseteq t$. 

\vspace{0.5mm}\noindent{\em Case 2}: $x\in B$. 
Then, $\delta_x=\top^+$ by the choice of $\delta$.
Condition~1 of Definition~\ref{df:canonical M} follows from line~2 of statement~\eqref{eq:10-July-3} and $X\subseteq t$.  Conditions 2 and 3 are satisfied because $\top\in t$. Condition 4 is satisfied because $\delta_x=\top^+\notin\Delta_x^s\setminus D^s_x$. 

\vspace{0.5mm}\noindent{\em Case 3}: $x\in C$ and $\delta_x=\top^+$. Condition 1 follows from item 3(a) of the choice of profile $\delta$. Conditions 2, 3, and 4 are similar to the previous case.

\vspace{0.5mm}\noindent{\em Case 4}: $x\in C$ and $\delta_x=\top^-$. Conditions~1, 2 and 3 are similar to case~1. Condition 4 follows from item 3(b) of the choice of profile $\delta$. 
\end{proof-of-claim}

\vspace{-2mm}
As the ending of the proof of the lemma, first, note that $\delta_a=\top^-\in \Delta_a^s\setminus D^s_a$ by the choice of profile $\delta$ and Definition~\ref{df:canonical action space & D}. Second, note that $\phi\in X\subseteq t$ by line~1 of statement~\eqref{eq:10-July-3} and the formation of set $t$.
\end{proof}

\subsection{Proof of Lemma~\ref{lm:truth lemma}}\label{app:core prove of truth lemma}

{\bf Lemma~\ref{lm:truth lemma}} {\em $s\Vdash \phi$ if and only if $\phi\in s$ for each state $s$ of the canonical model and each formula $\phi\in\Phi$.}
\vspace{1mm}

\begin{proof}
We prove the lemma by induction on the structural complexity of formula $\phi$. If $\phi$ is a propositional variable, then the statement of the lemma follows from item~\ref{item:sat propositional} of Definition~\ref{df:sat} and Definition~\ref{df:canonical pi}. If formula $\phi$ is a negation or a disjunction, then the statement of the lemma follows from the induction hypothesis, items~\ref{item:sat negation} and \ref{item:sat disjunction} of Definition~\ref{df:sat} and the maximality and consistency of set $s$ in the standard way.

\vspace{0.5mm}
Suppose that formula $\phi$ has the form $\WA_a\psi$. 

\vspace{0.4mm}
\noindent$\Rightarrow:$
Assume that $s\Vdash\WA_a\psi$. Then, 
by item~\ref{item:sat WA} of Definition~\ref{df:sat},
there exists an action $i\in D^s_a$ such that
$(s,i)\nstit_a\neg\psi$. Hence, by Definition~\ref{df:sat}, there is a tuple $(\delta,t)\in M_s$ such that $\delta_a=i$ and $t\nVdash\neg\psi$. Then, $t\Vdash\psi$ by item~\ref{item:sat negation} of Definition~\ref{df:sat}.
Thus, $\psi\in t$ by the induction hypothesis.  Hence, $\neg\phi\notin t$ because set $t$ is consistent.
Note that $\neg\phi\notin t$ and $\delta_a=i\in D^s_a$. Therefore, $\WA_a\psi\in s$ by Condition~1 of Definition~\ref{df:canonical M}.

\vspace{0.4mm}
\noindent$\Leftarrow:$ Assume that $\WA_a\psi\in s$. Then, by Lemma~\ref{lm:auxiliar WA}, 
there is a tuple $(\delta,t)\in M_s$ such that $\delta_a\in D^s_a$ and $\psi\in t$. Hence, $t\Vdash\psi$ by the induction hypothesis. Thus, $t\nVdash\neg\psi$ by item~\ref{item:sat negation} of Definition~\ref{df:sat}.
Then, $(s,\delta_a)\nstit_a\neg\psi$ by Definition~\ref{df:sat}.
Therefore, $s\Vdash\WA_a\psi$ by item~\ref{item:sat WA} of Definition~\ref{df:sat}.

\vspace{0.5mm}
Suppose that formula $\phi$ has the form $\WE_a\psi$. 

\vspace{0.4mm}
\noindent$\Rightarrow:$
Assume that $\WE_a\psi\notin s$. Then, $\neg\WE_a\psi\in s$ because $s$ is a maximal consistent set. Hence, by Lemma~\ref{lm:auxiliar WE}, for each action $i\in D_a^s$, there is a tuple $(\delta,t)\in M_s$ such that $\delta_a=i$ and $\neg\psi\in t$. Then, $\psi\notin t$ because $t$ is a maximal consistent set. Thus, $t\nVdash\psi$ by the induction hypothesis. Hence, $(s,i)\nstit_a\psi$ for each action $i\in D_a^s$ by Definition~\ref{df:sat}. Therefore, $s\nVdash\WE_a\psi$ by item~\ref{item:sat WE} of Definition~\ref{df:sat}.

\vspace{0.4mm}
\noindent$\Leftarrow:$
Assume that $s\nVdash\WE_a\psi$. Then, $(s,\psi^+)\nstit_a\psi$ by item~\ref{item:sat WE} of Definition~\ref{df:sat}. Thus, there is a tuple $(\delta,t)\in M_s$ such that $\delta_a=\psi^+$ and $t\nVdash\psi$ by Definition~\ref{df:sat}. Then, $\psi\notin t$ by the induction hypothesis. Note that $\delta_a\!=\!\psi^+$ and $\psi\notin t$. Therefore, $\WE_a\psi\notin s$ by Condition~2 of Definition~\ref{df:canonical M}.

\vspace{0.5mm}
Suppose that formula $\phi$ has the form $\SE_a\psi$. 

\vspace{0.4mm}
\noindent$\Rightarrow:$
Assume that $s\Vdash\SE_a\psi$. Then, $(s,\psi^-)\nstit_a\psi$ by item~\ref{item:sat SE} of Definition~\ref{df:sat} because $\psi^-\notin D_a^s$. Thus, there is a tuple $(\delta,t)\in M_s$ such that $\delta_a=\psi^-$ and $t\nVdash\psi$ by Definition~\ref{df:sat}. Then, $\psi\notin t$ by the induction hypothesis. Thus, $\delta_a=\phi^-$ and $\psi\notin t$. Therefore, $\SE_a\psi\in s$ by Condition~3 of Definition~\ref{df:canonical M}.

\vspace{0.4mm}
\noindent$\Leftarrow:$
Assume that $\SE_a\psi\in s$. Then, by Lemma~\ref{lm:auxiliar SE}, for each action $i\in\Delta_a^s\setminus D_a^s$, there is a tuple $(\delta,t)\in M_s$ such that $\delta_a=i$ and $\neg\psi\in t$. Thus, $\psi\notin t$ because $t$ is a maximal consistent set. Then, $t\nVdash\psi$ by the induction hypothesis. Hence, $(s,i)\nstit_a\psi$ for each action $i\in\Delta_a^s\setminus D_a^s$ by Definition~\ref{df:sat}. Then, by contraposition, $i\in D_a^s$ for each action $i$ such that $(s,i)\stit_a\psi$. Therefore, $s\Vdash\SE_a\psi$ by item~\ref{item:sat SE} of Definition~\ref{df:sat}.

\vspace{0.5mm}
Suppose that formula $\phi$ has the form $\SA_a\psi$. 

\vspace{0.4mm}
\noindent$\Rightarrow:$
Assume that $\SA_a\psi\notin s$. Then, $\neg\SA_a\psi\in s$ because $s$ is a maximal consistent set. Thus, by Lemma~\ref{lm:auxiliar SA}, there is a tuple $(\delta,t)\in M_s$ such that $\delta_a\in\Delta_a^s\setminus D_a^s$ and $\psi\in t$. Then, by the induction hypothesis, $t\Vdash\psi$. Thus, $t\nVdash\neg\psi$ by item~\ref{item:sat negation} of Definition~\ref{df:sat}. Hence, $(s,\delta_a)\nstit_a\neg\psi$ by Definition~\ref{df:sat}.  Therefore, $s\nVdash\SA_a\psi$ by item~\ref{item:sat SA} of Definition~\ref{df:sat} and because $\delta_a\in\Delta_a^s\setminus D_a^s$.

\vspace{0.4mm}
\noindent$\Leftarrow:$
Assume that $s\nVdash\SA_a\psi$. Then, by item~\ref{item:sat SA} of Definition~\ref{df:sat}, there is an action $i\in \Delta_a^s\setminus D_a^s$ such that $(s,i)\nstit_a\neg\psi$. Thus, there is a tuple $(\delta,t)\in M_s$ such that $\delta_a=i$ and $t\nVdash\neg\psi$ by Definition~\ref{df:sat}. Then, $t\Vdash\psi$ by item~\ref{item:sat negation} of Definition~\ref{df:sat}. Hence, $\psi\in t$ by the induction hypothesis. Then, $\neg\psi\notin t$ because $t$ is a maximal consistent set. Note that $\delta_a=i\in \Delta_a^s\setminus D_a^s$ and $\neg\psi\notin t$. Hence, $\SA_a\psi\notin s$ by Condition~4 of Definition~\ref{df:canonical M}.
\end{proof}

\subsection{Proof of Theorem~\ref{th:completeness}}\label{app:completeness theorem}

\noindent\textbf{Theorem~\ref{th:completeness}}
\textit{For each set of formulae $X\subseteq \Phi$ and each formula $\phi\in\Phi$ such that $X\nvdash\phi$, there is a state $s$ of a transition system such that $s\Vdash\chi$ for each $\chi\in X$ and $s\nVdash\phi$.}

\begin{proof}
Suppose that $X\nvdash \phi$. Then, the set $X\cup \{\neg\phi\}$ is consistent. According to Lemma~\ref{lm:Lindenbaum}, there is a maximal consistent extension $s$ of the set $X\cup \{\neg\phi\}$. Then, $s\in S$ by Definition~\ref{df:canonical S}. Note that $\chi\in s$ for each $\chi\in X$ because $X\subseteq s$. Also, $\phi\notin s$ due to the consistency of set $s$. Hence, $s\Vdash\chi$ for each $\chi\in X$ and $s\nVdash\phi$ by Lemma~\ref{lm:truth lemma}.
\end{proof}

\bibliographystyle{named}

\end{document}